\documentclass{article} % For LaTeX2e

\usepackage[accepted]{icml2018}
\usepackage{hyperref}
\usepackage{url}
\usepackage{epsfig}
\usepackage{amssymb}
\usepackage[tbtags]{amsmath}
\usepackage{amsthm}
\usepackage{graphics,eepic,epic}
\usepackage{latexsym}
\usepackage{euscript}
\usepackage{subfigure}
\usepackage{graphics,eepic,epic,psfrag}
\usepackage{bm}
\usepackage{dsfont}
\usepackage{algorithm}
\usepackage{wrapfig}
\usepackage{tikz}
\usetikzlibrary{positioning}
\usetikzlibrary{shapes}
\usepackage{listings}
\usepackage{color}
\usepackage{setspace}
\usepackage{caption}
\usepackage{booktabs}
\usepackage{nicefrac}

\definecolor{mydarkblue}{rgb}{0,0.08,0.45}
\hypersetup{colorlinks=true,
 linkcolor=mydarkblue,
 citecolor=mydarkblue,
 filecolor=mydarkblue,
 urlcolor=mydarkblue}

\newcommand{\argmin}{\mathop{\mathrm{argmin}}\limits}

\newtheorem{theorem}{Theorem}[section]

\newcommand{\prior}[1]{p \left( #1 \right)} % The prior belief distribution of a variable.(1) A parameter or hyper parameter.
\newcommand{\param}{\mathrm{w}} %The model parameters.
\newcommand{\paramFixed}{\param} %A fixed model parameter.
\newcommand{\hyper}{\lambda} %A variable hyper parameter.
\newcommand{\hyperFixed}{\hyper} %A fixed hyper parameter.
\newcommand{\hyperDist}{\prior{\hyper}} %The distribution of hyper parameters.
\newcommand{\innerOptParam}[1]{\param^{*} \! \left( #1 \right)} %The optimal parameters of the model for the inner optimization.(1) Variable hyper parameters or fixed hyper parameters.
\newcommand{\optParam}[1]{\param^{**}}% \! \left( #1 \right)} %The optimal parameters of the model after specifying the hyper parameter.(1) Variable hyper hyper parameters or fixed hyper hyper parameters.
\newcommand{\optHyper}[1]{\hyper^{*}}% \! \left( #1 \right)} %The optimal hyper parameter according to loss and the hyper parameter.(1) A variable or fixed hyper hyper parameter.
\newcommand{\lossSymbol}{\mathop{\mathcal{L}}} %The symbol to use for loss functions.
\newcommand{\lossSymbolInner}{\lossSymbol_{\mathrm{Train}}} %The symbol to use for the inner loss functions.
\newcommand{\lossSymbolOuter}{\lossSymbol_{\mathrm{Valid.}}} %The symbol to use for the outer loss functions.
\newcommand{\innerLoss}[2]{\lossSymbolInner \! \left( #1, #2 \right)} %The loss function for the inner optimization.(1) The model parameters or the hypernets output.(2) A variable hyper parameter, or a fixed hyper parameter.
\newcommand{\innerOpt}{\argmin_{\param} \innerLoss{\param}{\hyper}} %The result of the inner optimization.
\newcommand{\outerLoss}[1]{\lossSymbolOuter \! \left( #1 \right)}%, \hyperHyper \right)} %The loss function for the outer optimization.(1) The optimal model parameters, or the inner optimization.
\newcommand{\outerOpt}[1]{\argmin_{\hyper} \outerLoss{#1}} %The result of the inner optimization.(1) The optimal model parameters, or the inner optimization.
\newcommand{\predictionLoss}[2]{\lossSymbol_{\mathrm{Pred}} ( #1, #2 )} %The loss from predictions of the model.(1) Training or Testing or data point.
\newcommand{\regLoss}[2]{\lossSymbol_{\mathrm{Reg}} ( #1, #2 )}%The loss for the regularization of the model.(1) Hyper parameters or hyper hyper parameters.This is typically introduced in a principled fashion with MAP.
\newcommand{\outerUpdateSymbol}{\textrm{hyperopt}} %The symbol for the outer update function.
\newcommand{\outerUpdate}[1]{\outerUpdateSymbol \! \left( #1 \right)} %The function for updating the outer optimization parameter.
\newcommand{\variableData}{\bf{x}} % A variable data point
\newcommand{\ETrain}[1]{\mathop{\mathbb{E}}_{\variableData \sim \mathrm{Train}} \! \left[ #1 \right]}
\newcommand{\EValid}[1]{\mathop{\mathbb{E}}_{\variableData \sim \mathrm{Valid.}} \! \left[ #1 \right]}
\newcommand{\innerLossEExpand}[2]{\ETrain{\predictionLoss{\variableData}{#1}} + \regLoss{#1}{#2}} % The expected expanded training loss.(1) The model parameters that are fixed of hypernet learned.(2) The hyperparamter.
\newcommand{\outerLossEExpand}[1]{\EValid{\predictionLoss{\variableData}{#1}}} % The expected expanded validation loss.(1) The hyperparamter.
\newcommand{\outerIter}{T_{\mathrm{outer}}} % The number of iterations for the outer optimization.
\newcommand{\responseParam}{\phi} %The parameters of the response function.
\newcommand{\responseParamFixed}{\responseParam} %The fixed parameters of the response function.
\newcommand{\approxResponseSymbol}[1]{\param_{#1}} %The symbol for the approximate response function.(1) A variable hypernet, or a fixed hypernet.
\newcommand{\approxResponse}[2]{\approxResponseSymbol{#2} ( #1 )} %The approximate response function.(1) A variable hyperparameter or fixed hyperparameter.(2) A variable hypernet, or a fixed hypernet.
\newcommand{\sampleRename}[1]{#1} %Rename the sampled lambdas for algLocal
\newcommand{\curRename}[1]{\smash{\hat{#1}}} %Rename the current lambdas for algLocal
\newcommand{\hyperDistVar}{p ( \sampleRename{\hyper} | \curRename{\hyper} )} %The distribution of hyper parameters.
\newcommand{\argminTargetFix}{\responseParam}%\approxResponse{\hyper}{\responseParam}}

\newcommand{\approxResponseOutput}[1]{\approxResponseSymbol{\responseParam^{*}} ( #1 )}

\newcommand{\proofLoss}{\innerLoss{\approxResponse{\hyper}{\responseParam}}{\hyper}}

\newcommand{\proofTargetLossOutput}{\outerLoss{\approxResponseOutput{\hyper}}}
\newcommand{\targetLoss}{\outerLoss{\innerOptParam{\hyper}}}
\newcommand{\phyper}{p \left( \hyper \right)}
\newcommand{\hyperSupport}{\mathrm{support} \! \left( \phyper \right)} %\phyper
\newcommand{\hyperDomain}{\textnormal{for all } \hyper \in \hyperSupport}
\newcommand{\Ehyper}[1]{\mathop{\mathbb{E}}_{\phyper} \! \left[ #1 \right]}
\newcommand{\rename}[1]{#1'}
\newcommand{\EhyperFix}[1]{\mathop{\mathbb{E}}_{p \left( \rename{\hyper} \right)} \! \left[ #1 \right]}
\newcommand{\proofLossFix}{\innerLoss{\approxResponse{\rename{\hyper}}{\responseParam}}{\rename{\hyper}}}

\newcommand{\lossTrainData}[2]{\lossSymbolInner ( #1, #2, \variableData )} %The loss from predictions of the model.(1) Training or Testing or data point.
\newcommand{\lossValidData}[1]{\lossSymbolOuter ( #1, \variableData)} %The loss from predictions of the model.(1) Training or Testing or data point.
\newcommand{\lossValidDataChange}[2]{\lossSymbolOuter ( #1, #2)}

\begin{document}
\twocolumn[
\icmltitle{Stochastic Hyperparameter Optimization through Hypernetworks}
\begin{icmlauthorlist}
\icmlauthor{Jonathan Lorraine}{to}
\icmlauthor{David Duvenaud}{to}
\end{icmlauthorlist}
\icmlaffiliation{to}{Department of Computer Science, University of Toronto, Toronto, Canada}
\icmlcorrespondingauthor{Jonathan Lorraine}{lorraine@cs.toronto.edu}
\icmlkeywords{Machine Learning, Meta-learning}
]
\printAffiliationsAndNotice{}

\begin{abstract}
Machine learning models are often tuned by nesting optimization of model weights inside the optimization of hyperparameters.
We give a method to collapse this nested optimization into joint stochastic optimization of weights and hyperparameters.
Our process trains a neural network to output approximately optimal weights as a function of hyperparameters.
We show that our technique converges to locally optimal weights and hyperparameters for sufficiently large hypernetworks.
We compare this method to standard hyperparameter optimization strategies and demonstrate its effectiveness for tuning thousands of hyperparameters.
\end{abstract}

\section{Introduction}
Model selection and hyperparameter tuning is a significant bottleneck in designing predictive models.
Hyperparameter optimization is a nested optimization:
The inner optimization finds model parameters $\param$ which minimize the training loss $\lossSymbolInner$ given hyperparameters $\hyper$.
The outer optimization chooses $\hyper$ to reduce a validation loss $\lossSymbolOuter$:
\begin{figure}
\centering
\begin{tikzpicture}
    \node[shape=circle,draw=black] (optParams) at (-0.5,0.1) {$\alpha$};
    \node[shape=circle,draw=black, fill=gray] (x) at (-0.5,1) {$x$};
    \node[shape=circle,draw=black, fill=gray] (t) at (-0.5,2) {$t$};
    \node[shape=circle,draw=black] (lambda) at (1,-0.7) {$\hyper$};
    \node[shape=rectangle,draw=black] (train) at (1,0.1) {Training};
    \node[shape=circle,draw=black] (w) at (1,1) {$\param^{*}$};
    \node[shape=circle,draw=black] (y) at (1,2) {$y$};
    \node[shape=rectangle,draw=black] (L) at (1,2.8) {$\lossSymbolOuter$};
    
    \path [->] (optParams) edge node {} (train);
    \path [->] (lambda) edge node {} (train);
    \path [->] (train) edge node {} (w);
    \path [->] (x) edge node {} (y);
    \path [->] (w) edge node {} (y);
    \path [->] (t) edge node {} (L);
    \path [->] (y) edge node {} (L);
    
    \node[shape=circle,draw=black] (phi) at (2.5,0.1) {{\color{red}$\responseParam$}};
    \node[shape=circle,draw=black, fill=gray] (xNEW) at (2.5,1) {$x$};
    \node[shape=circle,draw=black, fill=gray] (tNEW) at (2.5,2) {$t$};
    \node[shape=circle,draw=black] (lambdaNEW) at (4.5,-0.7) {$\hyper$};
    \node[shape=rectangle,draw=black] (hypernet) at (4.5,0.1) {{\color{red}Hypernetwork}};
    \node[shape=circle,draw=black] (wNEW) at (4.5,1) {$\param^{*}$};
    \node[shape=circle,draw=black] (yNEW) at (4.5,2) {$y$};
    \node[shape=rectangle,draw=black] (LNEW) at (4.5,2.8) {$\lossSymbolOuter$};
    
    \path [->] (phi) edge node {} (hypernet);
    \path [->] (lambdaNEW) edge node {} (hypernet);
    \path [->] (hypernet) edge node {} (wNEW);
    \path [->] (xNEW) edge node {} (yNEW);
    \path [->] (wNEW) edge node {} (yNEW);
    \path [->] (tNEW) edge node {} (LNEW);
    \path [->] (yNEW) edge node {} (LNEW);
    \node[above] at (current bounding box.north) {\,\,\,\,\,\,\,\,Cross-validation \,\,\,\,\,\,\,\,\,\,\,\,\,\,\,\,\,\,\,\,\,Hyper-training};
\end{tikzpicture}
\caption{%A standard computational graph is shown in black, while our additions are in {\color{red} red}.
%Note that $t$ is the target for some data point $x$, which can be sampled from the training or validation sets.
%We want to make predictions $y$ for each data point that are close to the target, with closeness being determined by our loss $\lossSymbol$ and hyperparameters $\hyper$.
%The prediction is parameterized by $\param$ which is optimized to minimize the training loss, while the hyperparameters are optimized to minimize validation loss.
%This makes the optimal parameters dependent on the hyperparameters and the optimal hyperparameters dependent on the optimal parameters.
%We approximate this relationship with a hypernet parameterized by $\responseParam$ and leverage differentiability of the hypernetwork for hyperparameter optimization.
\emph{Left:} A typical computational graph for cross-validation, where $\alpha$ are the optimizer parameters, and $\hyper$ are training loss hyperparameters.
It is expensive to differentiate through the entire training procedure.
% - see \citet{maclaurin2015gradient}.
\emph{Right:} The proposed computational graph with our changes in {\color{red} red}, where $\responseParam$ are the hypernetwork parameters.
We can cheaply differentiate through the hypernetwork to optimize the validation loss $\lossSymbolOuter$ with respect to hyperparameters $\hyper$.
We use $x$, $t$, and $y$ to refer to a data point, its label, and a prediction respectively.}
\end{figure}
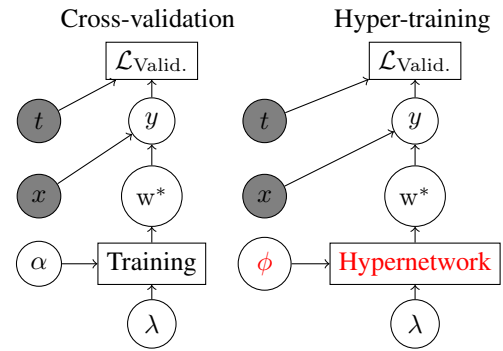
\begin{figure}
	%\vspace{-1cm}
	\centering
	\includegraphics[width=0.42\textwidth]{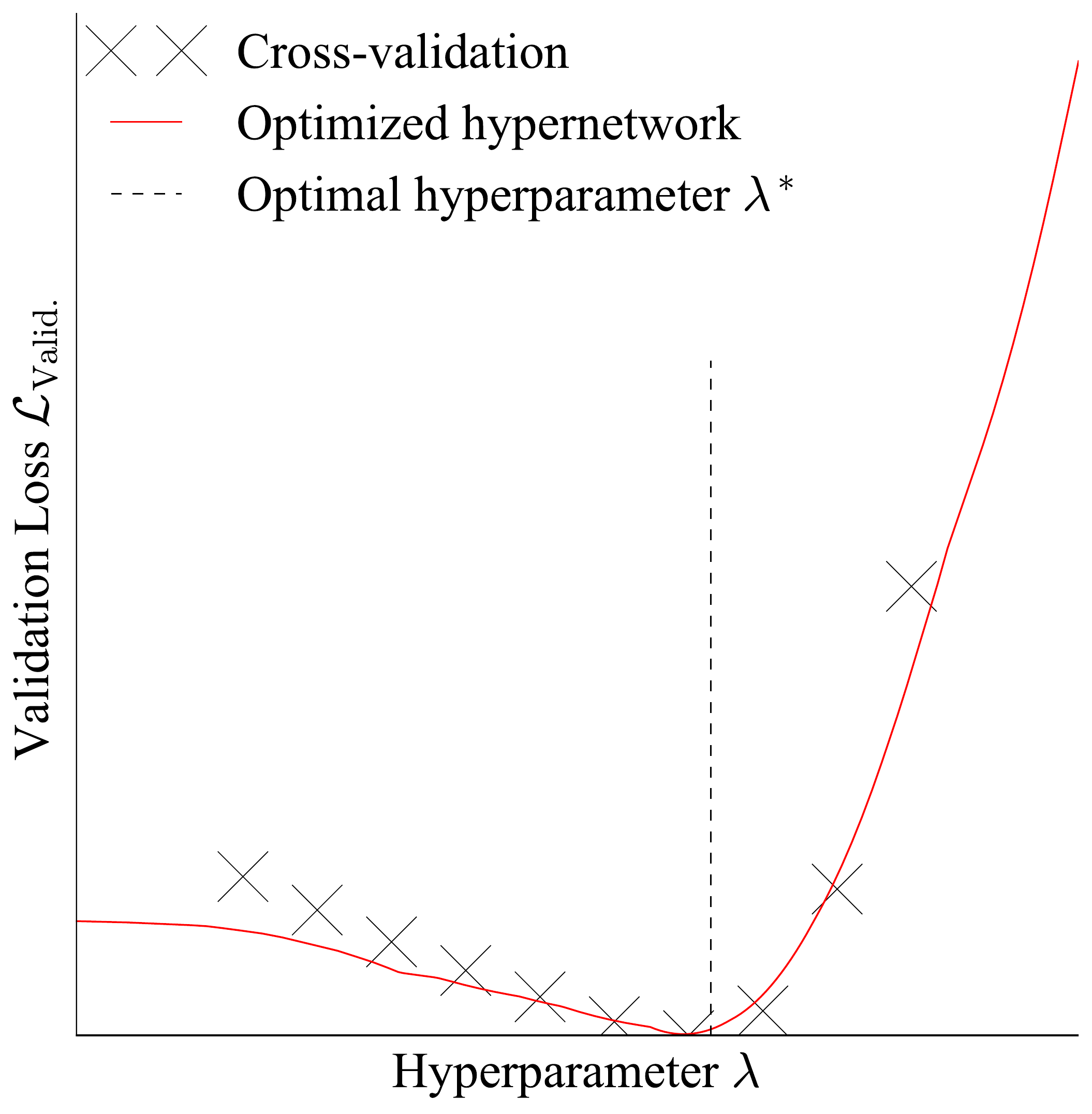}
	\caption{
	The validation loss of a neural net, estimated by cross-validation (crosses) or by a hypernetwork (line), which outputs $7,850$-dimensional network weights.
	Cross-validation requires optimizing from scratch each time.
	The hypernetwork can be used to evaluate the validation loss cheaply.
	\label{fig:exp1}
	}
\end{figure}
\begin{align}
\outerOpt{\innerOpt}
\label{eq nested}
\end{align}
Standard practice in machine learning solves~\eqref{eq nested} by gradient-free optimization of hyperparameters, such as grid search or random search.
%There are also methods using gradients of surrogate functions like Bayesian optimization.
Each set of hyperparameters is evaluated by re-initializing weights and training the model to completion.
Re-training a model from scratch is wasteful if the hyperparameters change by a small amount.
Some approaches, such as Hyperband~\citep{li2016hyperband} and freeze-thaw Bayesian optimization~\citep{swersky2014freeze}, resume model training and do not waste this effort.
However, these methods often scale poorly beyond 10 to 20 dimensions.

How can we avoid re-training from scratch each time?
Note that the optimal parameters $\param$ are a deterministic function of the hyperparameters $\hyper$:
\begin{align}
\innerOptParam{\hyper} = \innerOpt
\label{best response equation}
\end{align}
We propose to \emph{learn this function}.
Specifically, we train a neural network that takes hyperparameters as input, and outputs an approximately optimal set of weights.

This formulation provides two major benefits:
First, we can train the hypernetwork to convergence using stochastic gradient descent (SGD) without training any particular model to completion.
Second, differentiating through the hypernetwork allows us to optimize hyperparameters with stochastic gradient-based optimization.
%
%\section{Background}
%The validation loss gives us an estimate of the performance on unseen data, while the training loss gives us an estimate of performance on observed data.
%We want to find models which perform well on unobserved data given observed data, so optimization of model parameters for training loss is nested in an optimization of models for validation loss.

\begin{figure*}
	\centering
	\begin{minipage}{1.0\textwidth}
	\begin{tabular}{c|l}
\phantom{aaaaaaaaaaaaaa} Training loss surface & \phantom{aaaaaaaa} Validation loss surface\\
\includegraphics[width=0.5\textwidth, clip, trim=10mm 0mm 0mm 15mm]{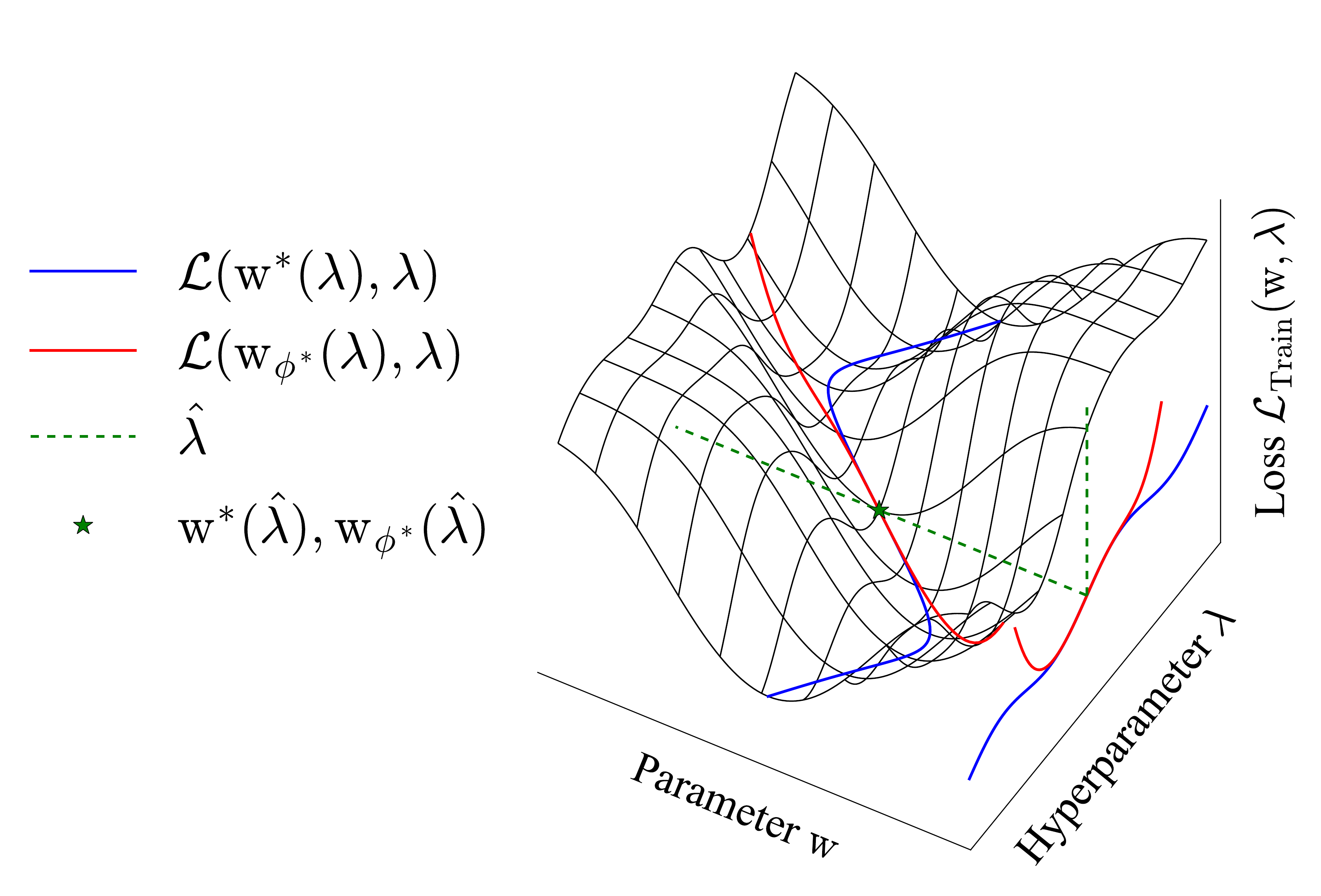} &
\includegraphics[width=0.5\textwidth, clip, trim=35mm 0mm 0mm 18mm]{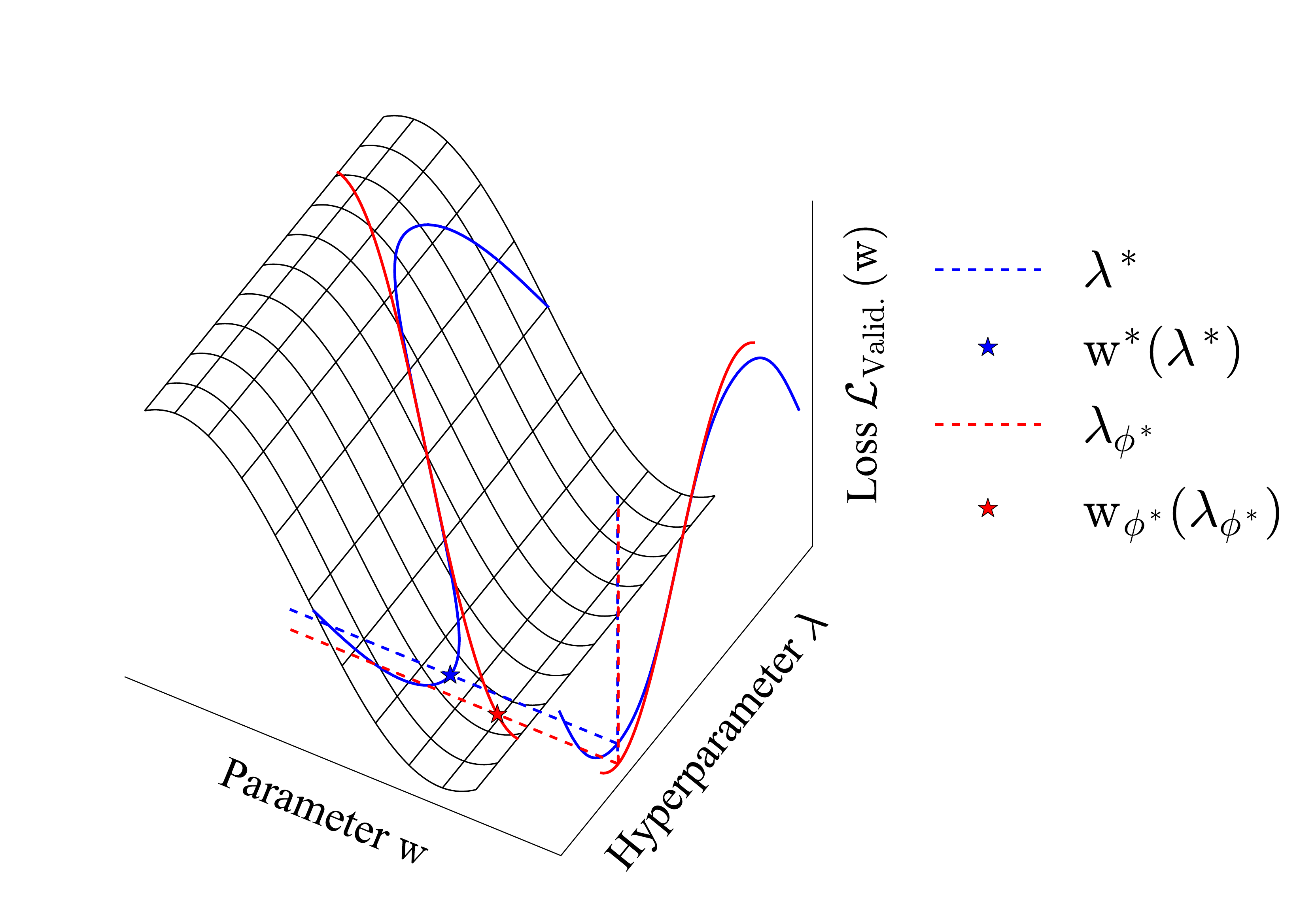}
\end{tabular}
\end{minipage}
\caption{
A visualization of exact (blue) and approximate (red) optimal weights as a function of hyperparameters.
%	\emph{Left:} The training loss surface.
%	\emph{Right:} The validation loss surface.
The approximately optimal weights $\param_{\phi^{*}}$ are output by a linear model fit at $\curRename{\lambda}$.
The true optimal hyperparameter is $\lambda^{*}$, while the hyperparameter estimated using approximately optimal weights is nearby at $\lambda_{\phi^{*}}$.
\label{fig:theory1}
}
\end{figure*}

\section{Training a network to output optimal weights}
\label{sec.new_formulation}
How can we teach a \emph{hypernetwork}~\citep{ha2016hypernetworks} to output approximately optimal weights to another neural network?
The basic idea is that at each iteration, we ask a hypernetwork to output a set of weights given some hyperparameters:
$\approxResponseSymbol{} = \approxResponse{\hyperFixed}{\responseParamFixed}$.
Instead of updating the weights $\param$ using the training loss gradient $\nicefrac{\partial \lossSymbolInner (\param)}{\partial \param}$, we update the hypernetwork weights $\responseParamFixed$ using the chain rule: $\frac{\partial \lossSymbolInner(\param_\responseParamFixed)}{\partial \param_\responseParamFixed} \frac{\partial \param_\responseParamFixed}{\partial \responseParamFixed}$.
This formulation allows us to optimize the hyperparameters $\hyper$ with the validation loss gradient $\frac{\partial \lossSymbolOuter(\param_\responseParamFixed (\hyper ))}{\partial \param_\responseParamFixed (\hyper )} \frac{\partial \param_\responseParamFixed (\hyper )}{\partial \hyper}$.
We call this method \emph{hyper-training} and contrast it with standard training methods.
%
%In game theory terms, it can be seen as a best-response function of a secondary player with strategy $\param$ to a leading player with strategy $\hyper$.
%\emph{Hypergradients} are gradients for updating hyperparameters.
%Our methods optimize hyperparameters with hypergradients through a hypernetwork, so we call them \emph{hyper-training}.

We call the function $\innerOptParam{\hyper}$ that outputs optimal weights for hyperparameters a \emph{best-response function}.
At convergence, we want our hypernetwork $\approxResponse{\hyperFixed}{\responseParamFixed}$ to match the best-response function closely.

Our method is closely related to the concurrent work of \citet{brock2017smash}, whose SMASH algorithm also approximates the optimal weights as a function of model architectures, to perform a gradient-free search over discrete model structures.
%The first section of Algorithm~\ref{algGlobal} is identical to the SMASH algorithm.
Their work focuses on efficiently estimating the performance of a variety of model architectures, while we focus on efficiently exploring continuous spaces of models.
We further extend this idea by formulating an algorithm to optimize the hypernetwork and hyperparameters jointly.
Joint optimization of parameters and hyperparameters addresses one of the main weaknesses of SMASH, which is that the the hypernetwork must be very large to learn approximately optimal weights for many different settings.
During joint optimization, the hypernetwork need only model approximately optimal weights for the neighborhood around the current hyperparameters, allowing us to use even linear hypernetworks.

%This section shows the inner optimization loop can be avoided by training a neural network to output approximately optimal weights for hyperparameters.
%Let $\hyperDist$ be a distribution of hyperparameters.
%A hypernetwork is constructed, $\approxResponse{\hyperFixed}{\responseParamFixed}$, parameterized by $\responseParam$ which takes hyperparameter $\hyper \sim \hyperDist$ as input and tries to output best-responding weights $\innerOptParam{\hyper}$.
%Locally optimal hyperparameters and weights are returned by Algorithm~\ref{algGlobal} if the hypernetwork mapping from hyperparameters to weights is a best-response ($\approxResponseSymbol{\responseParamFixed} = \innerOptParamSymbol$).
%We can do gradient descent on the weights returned by the hypernetwork to fine-tune them.

\subsection{Advantages of hypernetwork-based optimization}
Hyper-training is a method to learn a mapping from hyperparameters to validation loss which is differentiable and cheap to evaluate.
We can compare hyper-training to other model-based hyperparameter schemes.
Bayesian optimization (e.g., \citet{lizotte2008practical, snoek2012practical}) builds a model of the validation loss as a function of hyperparameters, usually using a Gaussian process (e.g., \citet{rasmussen2006gaussian}) to track uncertainty.
This approach has several disadvantages compared to hyper-training.

First, obtaining data for standard Bayesian optimization requires optimizing models from initialization for each set of hyperparameters.
In contrast, hyper-training never needs to optimize any one model fully removing choices like how many models to train and for how long.

Second, standard Bayesian optimization treats the validation loss as a black-box function: ${\mathcal{\hat \lossSymbolOuter(\lambda)} = f(\lambda)}$.
In contrast, hyper-training takes advantage of the fact that the validation loss is a known, differentiable function: ${\mathcal{\hat \lossSymbolOuter}(\lambda) = \mathcal{\lossSymbolOuter}(\approxResponse{\hyperFixed}{\responseParamFixed})}$.
This information removes the need to learn a model of the validation loss.
This function can also be evaluated stochastically by sampling points from the validation set.

Hyper-training has a benefit of learning hyperparameter to optimized weight mapping, which is substituted into the validation loss.
This often has a better inductive bias for learning hyperparameter to validation loss than directly learning the loss.
Also, the hypernetwork learns continuous best-responses, which may be a beneficial prior for finding weights by enforcing stability.
%A continuous best-response will guarantee the weights do not change much when the hyperparameter is modified slightly, which may be connected with generalization.

\subsection{Limitations of hypernetwork-based optimization}
We can apply this method to unconstrained continuous bi-level optimization problems with an inner loss function with inner parameters, and an outer loss function with outer parameters.
What sort of parameters can be optimized by our approach?
Hyperparameters typically fall into two broad categories:
1) Optimization hyperparameters, such as learning rates, which affect the choice of locally optimal point converged to, and
2) regularization or model architecture parameters which change the set of locally optimal points.
Hyper-training \emph{does not have inner optimization parameters} because there is no internal training loop, so we can not optimize these.
However, we must still choose optimization parameters for the fused optimization loop.
In principle, hyper-training can handle discrete hyperparameters, but does not offer particular advantages for optimization over continuous hyperparameters.

%A limitation is that we may need to sample a number of hyperparameters that is exponential in their dimensionality.
%We can only optimize continuous hyperparameters for the training loss - the architecture of a network is often discrete.
%The hyperparameters being optimized must affect the training loss - this excludes optimization hyperparameters like learning rate.
Another limitation is that our approach only proposes making local changes to the hyperparameters, and does not do uncertainty-based exploration.
Uncertainty can be incorporated into the hypernetwork by using stochastic variational inference as in \citet{blundell2015weight}, and we leave this for future work.
%Also, we do not optimize hyperparameters of the optimizer like learning rate and momentum.
%We want to have a way to guarantee the hypernetwork is a sufficiently close approximation each time hyperparameters are updated.
%Also, parameters for optimization (ex.optimizer parameters, network architecture, conditional hyperparameter distribution) can have a large impact on speed and stability of Algorithm~\ref{algLocal}.
%Finally, we may need a prohibitively large number of parameters in our hypernetwork.
Finally, it is not obvious how to choose the training distribution of hyperparameters $p(\hyper)$.
If we do not sample a sufficient range of hyperparameters, the implicit estimated gradient of the validation loss w.r.t.\ the hyperparameters may be inaccurate.
We discuss several approaches to this problem in section~\ref{joint optimization}.

A clear difficulty of this approach is that hypernetworks can require several times as many parameters as the original model.
For example, training a fully-connected hypernetwork with 1 hidden layer of $H$ units to output $D$ parameters requires at least $D \times H$ hypernetwork parameters.
To address this problem, in section~\ref{joint optimization}, we propose an algorithm that only trains a linear model mapping hyperparameters to model weights.
\begin{figure*}[t]
\begin{minipage}{0.355\textwidth}
\begin{algorithm}[H]
\begin{algorithmic}
\captionof{algorithm}{Standard cross-validation with stochastic optimization\label{alg1}}
\FOR{$i = 1, \dots, \outerIter$}
	\STATE $\textnormal{initialize } \paramFixed$
	\STATE $\hyperFixed = \outerUpdate{\hyperFixed^{(1:i)}, \outerLoss{\paramFixed^{(1:i)}}} \vphantom{\curRename{\hyper}}$
	\LOOP%FOR{$\innerIter$ steps}
		\STATE $\variableData \sim \textnormal{Training data} \vphantom{\hat{a^{i}}}$
		\STATE $\paramFixed \mathrel{-}= \alpha \nabla_{\param} \lossTrainData{\param}{\hyper}$
	\ENDLOOP%ENDFOR
	\STATE $\hyper^{i}, \param^{i} = \hyper, \param$
\ENDFOR
%\FOR{$i = 1, \dots, \outerIter$}
	%\IF {$\outerLoss{\paramFixed^{(i)}} < \outerLoss{\paramFixed}$ $\vphantom{\variableData \sim \mathrm{Validation \, Data} \vphantom{\hyperFixed = \outerUpdate{\dots, \hyperFixed^{(i)}, \outerLoss{\paramFixed^{(i)}}}}}$}
	%\STATE $\curRename{\hyper}, \param = \hyper^{i}, \param^{i}$ $\vphantom{\curRename{\hyperFixed} = \curRename{\hyperFixed} - \beta \nabla_{\curRename{\hyper}} \outerLossExpand{\approxResponse{\curRename{\hyper}}{\responseParamFixed}}}$
	%\ENDIF
	\STATE 
	\STATE $i = \argmin_{i} \lossTrainData{\param^{(i)}}{\hyper^{(i)}}$
	\STATE 
%\ENDFOR 
\STATE Return $\hyperFixed^{(i)}, \param^{(i)}$
\end{algorithmic}
\end{algorithm}
\end{minipage}
\begin{minipage}{0.33\textwidth}
\begin{algorithm}[H]
	\begin{algorithmic}
	\captionof{algorithm}{Optimization of hypernetwork, then hyperparameters \label{algGlobal}}
	\STATE $\vphantom{i = 1, \dots, \outerIter}$
	\STATE $\textnormal{initialize } {\color{blue}\responseParamFixed}$
	\STATE ${\color{blue}\textrm{initialize } \curRename{\hyper}}$
	\LOOP%FOR{$T_{\mathrm{{\color{blue}hypernetwork}}}$ steps}
		\STATE $\variableData \sim \textnormal{Training data}${\color{blue}, $\hyperFixed \sim \hyperDist$}
		\STATE ${\color{blue}\responseParamFixed} \mathrel{-}= \alpha \nabla_{{\color{blue}\responseParam}}\lossTrainData{ \approxResponseSymbol{{\color{blue}\responseParam}} {\color{blue} ( \hyperFixed )}}{\hyper}$
	\ENDLOOP%ENDFOR
	\STATE $\vphantom{\hyper^{i}, \param^{i} = \hyper, \param}$
	\LOOP%FOR{{\color{blue}$T_{\mathrm{hyperparameter}}$ steps}}
		\STATE ${\color{blue}\variableData \sim \textnormal{Validation data} \vphantom{\outerLoss{\paramFixed^{(i)}} < \outerLoss{\paramFixed}}}$
		\STATE ${\color{blue}\curRename{\hyperFixed} \mathrel{-}= \beta \nabla_{\curRename{\hyper}} \lossValidData{\approxResponse{\curRename{\hyper}}{\responseParamFixed}}}$
	\ENDLOOP%\\%ENDFOR\\
	\STATE Return $\curRename{\hyperFixed}, \approxResponseSymbol{{\color{blue}\responseParam}} {\color{blue} ( \curRename{\hyperFixed} )}$
	\end{algorithmic}
\end{algorithm}
\end{minipage}
\begin{minipage}{0.33\textwidth}
\begin{algorithm}[H]
	\begin{algorithmic}
	\captionof{algorithm}{Joint optimization of hypernetwork and hyperparameters\label{algLocal}}
	\STATE $\vphantom{i = 1, \dots, \outerIter}$
	\STATE $\textrm{initialize } \responseParamFixed$
	\STATE $\textrm{initialize } \curRename{\hyperFixed}$
	%\STATE %$\vphantom{\curRename{\hyper}}$
	\LOOP%FOR{$T_{{\color{blue}\mathrm{joint}}}$ steps}
		\STATE $\variableData \sim \textnormal{Training data}$, $\sampleRename{\hyper} \sim p ( \sampleRename{\hyper} {\color{red}| \curRename{\hyper}} )$
		\STATE $\responseParamFixed \mathrel{-}= \alpha \nabla_{\responseParam} \lossTrainData{\approxResponse{\hyperFixed}{\responseParamFixed}}{\hyper}$
		\STATE 
		\STATE 
		\STATE $\variableData \sim \textnormal{Validation data}$
		\STATE $\curRename{\hyperFixed} \mathrel{-}= \beta \nabla_{\curRename{\hyper}} \lossValidData{\approxResponse{\curRename{\hyper}}{\responseParamFixed}}$
	\ENDLOOP%\\%ENDFOR \\
	\STATE Return $\curRename{\hyperFixed}, \approxResponse{\curRename{\hyperFixed}}{\responseParamFixed}$
	\end{algorithmic}
\end{algorithm}
\end{minipage}

\caption*{A comparison of standard hyperparameter optimization, our first algorithm, and our joint algorithm.
Here, $\outerUpdateSymbol$ refers to a generic hyperparameter optimization.
Instead of updating weights $\param$ using the loss gradient $\nicefrac{\partial \mathcal{L}(\param)}{\partial \param}$, we update hypernetwork weights $\responseParamFixed$ and hyperparameters $\hyper$ using the chain rule: $\frac{\partial \lossSymbolInner (\param_\responseParamFixed)}{\partial \param_\responseParamFixed} \frac{\partial \param_\responseParamFixed}{\partial \responseParamFixed}$ or $\frac{\partial \lossSymbolOuter (\param_\responseParamFixed (\hyper))}{\partial \param_\responseParamFixed (\hyper)} \frac{\partial \param_\responseParamFixed (\hyper) }{\partial \hyper}$ respectively.
This allows our method to use gradient-based hyperparameter optimization.
}
\label{compare with cv}
\end{figure*}

\subsection{Asymptotic convergence properties}
Algorithm~\ref{algGlobal} trains a hypernetwork using SGD, drawing hyperparameters from a fixed distribution $p(\lambda)$.
This section proves that Algorithm~\ref{algGlobal} converges to a local best-response under mild assumptions.
In particular, we show that, for a sufficiently large hypernetwork, the choice of $p(\lambda)$ does not matter as long as it has sufficient support.
Notation as if $\param_{\responseParam}$ has a unique solution for $\responseParam$ or $\param$ is used for simplicity, but is not true in general.
 
%Algorithm~\ref{algGlobal} solves for $\param_{\responseParam^{*}}$, where $\responseParam^{*} = \mathrm{argmin}_{\argminTargetVary} \Ehyper{\proofLoss}$.
%We show $\hyperDomain, \proofTargetLossOutput = \targetLoss$ to prove a suitable surrogate is learned for hyperparameter optimization.
%Notation as if $\param_{\responseParam}$ has a unique solution to $\responseParam$ or $\param$ is used for simplicity, but is not true in general.
%
\begin{theorem}
\label{amoritizedExactness}
Sufficiently powerful hypernetworks can learn continuous best-response functions, which minimizes the expected loss for all hyperparameter distributions with convex support.

\begin{align*}
&\textnormal{There exists } \responseParam^{*}, \textnormal{ such that } \hyperDomain,\\
&\innerLoss{\param_{\responseParam^{*}} \left( \hyper \right)}{\hyper} = \min_{\param} \innerLoss{\param}{\hyper} \textnormal{ and}\\
&\responseParam^{*} = \argmin_{\argminTargetFix} \EhyperFix{\proofLossFix}
\end{align*}
\end{theorem}
\begin{proof}
If $\param_{\responseParam}$ is a universal approximator~\citep{hornik1991approximation} and the best-response is continuous in $\hyper$ (which allows approximation by $\param_{\responseParam}$), then there exists optimal hypernetwork parameters $\responseParam^{*}$ such that for all hyperparameters $\hyper$, $\param_{\responseParam^{*}} (\hyper) = \mathrm{argmin}_{\param} \innerLoss{\param}{\hyper}$.
Thus, $\innerLoss{\param_{\responseParam^{*}} \left( \hyper \right)}{\hyper} = \min_{\param} \innerLoss{\param}{\hyper}$.
In other words, universal approximator hypernetworks can learn continuous best-responses.

Substituting $\responseParam^{*}$ into the training loss gives ${\Ehyper{\innerLoss{\approxResponse{\hyper}{\responseParam^{*}}}{\hyper}} = \Ehyper{\min_{\argminTargetFix} \proofLoss}}$.
By Jensen's inequality, $\min_{\argminTargetFix} \Ehyper{\proofLoss} \geq \Ehyper{\min_{\argminTargetFix} \proofLoss}$.
To satisfy Jensen's requirements, we have $\min_{\argminTargetFix}$ as our convex function on the convex vector space of functions $\{\proofLoss \textnormal{for } \hyper \in \hyperSupport \}$.
To guarantee convexity of the vector space we require that $\hyperSupport$ is convex and $\innerLoss{\param}{\hyper} = \innerLossEExpand{\param}{\hyper}$ with $\regLoss{\param}{\hyper} = \hyper \cdot \lossSymbol ( \param)$.
Thus, $\responseParam^{*} = \mathrm{argmin}_{\argminTargetFix} \Ehyper{\proofLoss}$.
In other words, if the hypernetwork learns the best-response it will simultaneously minimize the loss for every point in $\hyperSupport$.
\end{proof}
Thus, having a universal approximator and a continuous best-response implies $\hyperDomain$, $\proofTargetLossOutput = \targetLoss$, because $\approxResponseOutput{\hyper} = \innerOptParam{\hyper}$.
Thus, under mild conditions, we will learn a best-response in the support of the hyperparameter distribution.
If the best-response is differentiable, then there is a neighborhood about each hyperparameter where the best-response is approximately linear.
If the support of the hyperparameter distribution is this neighborhood, then we can learn the best-response locally with linear regression.
%Since $p$ was arbitrary up to support, the probability of sampling each point in the support does not matter.
%Algorithm~\ref{algGlobal} has the choice of hyperparameter distribution, $\hyperDist$, and this theorem says we only have to care about the support of $\hyperDist$ if our network is a universal approximator.

In practice, there are no guarantees about the network being a universal approximator or the finite-time convergence of optimization.
The optimal hypernetwork will depend on the hyperparameter distribution $p(\hyper)$, not just the support of this distribution.
We appeal to experimental results that our method is feasible in practice.

\begin{figure}
\vspace{-0.25cm}
\begin{algorithm}[H]
\begin{algorithmic}
	\captionof{algorithm}{Simplified joint optimization of hypernetwork and hyperparameters\label{algMin}}
	\STATE $\textrm{initialize } \responseParamFixed, \curRename{\hyperFixed}$
	\LOOP%FOR{$T_{{\color{blue}\mathrm{joint}}}$ steps}
		\STATE $\variableData \sim \textnormal{Training data},$ $\variableData' \sim \textnormal{Validation data}$
		\STATE $\responseParamFixed \mathrel{-}= \alpha \nabla_{\responseParam} \lossTrainData{\approxResponse{{\color{blue}  \curRename{\hyperFixed}}}{\responseParamFixed}}{{\color{blue}  \curRename{\hyperFixed}}}$
		\STATE $\curRename{\hyperFixed} \mathrel{-}= \beta \nabla_{\curRename{\hyper}} \lossValidDataChange{\approxResponse{\curRename{\hyper}}{\responseParamFixed}}{\variableData'}$
	\ENDLOOP%\\%ENDFOR \\
	\STATE Return $\curRename{\hyperFixed}, \approxResponse{\curRename{\hyperFixed}}{\responseParamFixed}$
	\end{algorithmic}
	\end{algorithm}
\caption*{Algorithm~\ref{algMin} builds on Algorithm~\ref{algLocal} by using gradient updates on $\curRename{\hyper}$ as a source of noise.
This variant does not have asymptotic guarantees, but performs similarly to Algorithm~\ref{algLocal} in practice.
%We can vectorize the updates to $\responseParam$ and $\curRename{\hyper}$.
}
\vspace{-1cm}
\end{figure}

\subsection{Jointly training parameters and hyperparameters}
\label{joint optimization}

\begin{figure}[ht]
\centering
\includegraphics[width=0.4\textwidth]{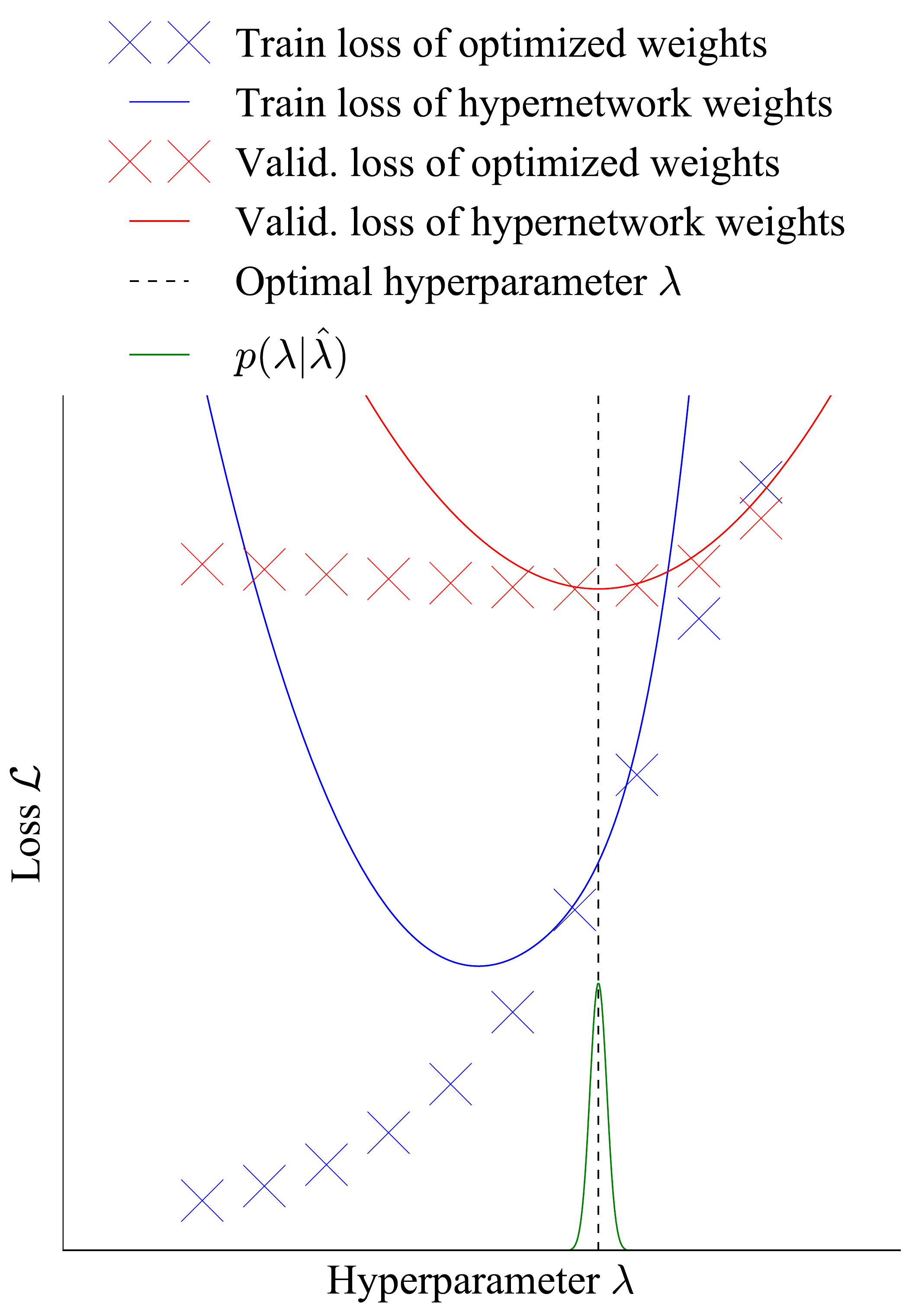}
\caption{
Training and validation losses of a neural network, estimated by cross-validation (crosses) or a linear hypernetwork (lines).
The hypernetwork's limited capacity makes it only accurate where  the hyperparameter distribution puts mass.
	\label{fig:exp2}}
\end{figure}

%Algorithm~\ref{algGlobal} has limitations.
%If there is limited training time, the hypernetwork's capacity may be limited.
%Thus, a hypernetwork with no (or few) hidden units is used.
Theorem~\ref{amoritizedExactness} holds for any $\hyperDist$.
In practice, we should choose a $\hyperDist$ that puts most of its mass on promising hyperparameter values, because it may not be possible to learn a best-response for all hyperparameters due to limited hypernetwork capacity.
%Also, it may be difficult to stabilize training when hyperparameters are sampled with a large support.
Thus, we propose Algorithm~\ref{algLocal}, which only tries to match a best-response locally.
We introduce a ``current'' hyperparameter $\curRename{\hyperFixed}$, which is updated each iteration.
We define a conditional hyperparameter distribution, $\hyperDistVar$, which only puts mass close to $\curRename{\hyperFixed}$.

Algorithm~\ref{algLocal} combines the two phases of Algorithm~\ref{algGlobal} into one.
Instead of first learning a hypernetwork that can output weights for any hyperparameter then optimizing the hyperparameters, Algorithm~\ref{algLocal} only samples hyperparameters near the current guess.
This means the hypernetwork just has to be trained to estimate good enough weights for a small set of hyperparameters.
There is an extra cost of having to re-train the hypernetwork each time we update $\curRename{\hyperFixed}$.
The locally-trained hypernetwork can then be used to provide gradients to update the hyperparameters based on validation set performance.
%This motivates the joint training in Algorithm~\ref{algLocal}, which optimizes hyperparameters more efficiently than Algorithm~\ref{algGlobal} by only learning how to update locally.

How simple can we make the hypernetwork, and still obtain useful gradients to optimize hyperparameters?
Consider the case in our experiments where the hypernetwork is a linear function of the hyperparameters and the conditional hyperparameter distribution is $\hyperDistVar = \mathcal{N} ( \curRename{\hyper}, \sigma \mathbb{I} )$ for some small $\sigma$.
This hypernetwork learns a tangent hyperplane to a best-response function and only needs to make minor adjustments at each step if the hyperparameter updates are sufficiently small.
We can further restrict the capacity of a linear hypernetwork by factorizing its weights, effectively adding a bottleneck layer with a linear activation and a small number of hidden units.
%In practice, trade-offs are made about the $\hyper$ distribution, the capacity of $\param_{\responseParam}$, and optimization parameters (batch size, learning rate, etc.).
%It is often true that as $\param_{\responseParam}$'s capacity is restricted, the $\hyper$ distribution must be densely centered on the current $\curRename{\hyper}$ to stabilize training.
%Also, as $\param_{\responseParam}$'s capacity is increased or the $\hyper$ distribution spreads out, we often must train longer.

\section{Related Work}
Our work is complementary to the SMASH algorithm of \citet{brock2017smash}, with section 2 discussing our differences.
%The first section of Algorithm~\ref{algGlobal} is identical to the SMASH algorithm.
%Their work focuses on efficiently evaluating the performance of a variety of discrete model architectures, while we focus on efficiently exploring continuous spaces of models.
%We also compare the ability of a hypernetwork to predict validation performance with standard black-box methods such as Gaussian processes.
% proposes builds on SMASH by doing stochastic gradient descent, denoted SGD, on hyperparameters after optimizing the hypernetwork.
%This as a generalization of Algorithm~\ref{alg1} where $\phyper = \delta_{\curRename{\hyper}}$ (a Dirac Delta distribution about some $\curRename{\hyper}$) and $\approxResponse{\hyper}{\responseParam} = \responseParam$ (weights not dependent on $\hyper$ because $\hyper$ can assume one value).

\paragraph{Model-free approaches}
Model-free approaches use only trial-and-error to explore the hyperparameter space.
Simple model-free approaches applied to hyperparameter optimization include grid search and random search~\citep{bergstra2012random}.
Hyperband~\citep{li2016hyperband} combines bandit approaches with modeling the learning procedure.

\paragraph{Model-based approaches}
Model-based approaches try to build a surrogate function, which can allow gradient-based optimization or active learning.
%Model-based approaches to hyperparameter learning can be gradient-based or facilitate active learning.
A common example is Bayesian optimization.
Freeze-thaw Bayesian optimization can condition on partially-optimized model performance.
%add modeling optimization to Bayesian optimization.

\paragraph{Optimization-based approaches}
Another line of related work attempts to directly approximate gradients of the validation loss with respect to hyperparameters.
\citet{domke2012generic} proposes to differentiate through unrolled optimization to approximate best-responses in nested optimization and \citet{maclaurin2015gradient} differentiate through entire unrolled learning procedures.
DrMAD~\citep{fu2016drmad} approximates differentiating through an unrolled learning procedure to relax memory requirements for deep neural networks.
HOAG~\citep{pedregosa2016hyperparameter} finds hyperparameter gradients with implicit differentiation by deriving an implicit equation for the gradient with optimality conditions.
\citet{franceschi2017forward} study forward and reverse-mode differentiation for constructing hyperparameter gradients.
Also, \citet{feng2017gradient} establish conditions where the validation loss of best-responding weights are almost everywhere smooth, allowing gradient-based training of hyperparameters.

%There are other methods of estimating gradients w.r.t.\ hyperparameters.
A closely-related procedure to our method is the $T1 - T2$ method of \citet{luketina2016scalable}, which also provides an algorithm for stochastic gradient-based optimization of hyperparameters.
The convergence of their procedure to local optima of the validation loss depends on approximating the Hessian of the training loss for parameters with the identity matrix.
In contrast, the convergence of our method depends on having a suitably powerful hypernetwork.

\paragraph{Game theory}
Best-response functions are extensively studied as a solution concept in discrete and continuous multi-agent games (e.g., \citet{fudenberg1998theory}).
Games where learning a best-response can be applied include adversarial training~\citep{goodfellow2014generative}, or Stackelberg competitions (e.g., \citet{bruckner2011stackelberg}).
For adversarial training, the analog of our method is a discriminator who observes the generator's parameters.
%Cross-validation is a method for model selection surveyed in \citet{arlot2010survey}, while \citet{kunapuli2008bilevel} consider the relationship between model selection and bi-level optimization problems.

\newcommand{\imSize}{28} %The number of pixels in an image length/width.
\newcommand{\numClass}{10} %The number of classes that images can belong to.
\newcommand{\numModelParams}{7,850} %This is \imSize \cdot \imSize \cdot \numClass + \numClass.
\newcommand{\regressionType}{linear } %The type of regression we are doing.
\newcommand{\numValidSmall}{10,000}
\newcommand{\numTestSmall}{10,000}
\newcommand{\numTrainSmall}{10} % The number of samples to use for the small training experiments.
\newcommand{\hyperDimSmall}{1} %The dimensionality of the hyperparameter space.
\newcommand{\realIter}{1,000} %The number of iterations to use when solving for the true model parameters.
\newcommand{\stepSizeReal}{0.0001} %The step size to use for finding the real loss.
\newcommand{\numHiddenGlobal}{50} %The number of hidden units in the hypernetwork.
\newcommand{\batchSizeGlobal}{2} %The batch size for each update (in terms of hyperparameters).
\newcommand{\stepSizeHypernetGlobal}{0.0001} %The step size to use when training the hyperparameter.
\newcommand{\hyperMeanGlobal}{0} %The mean of the hyperparameter distribution for sampling hyperparameters for training the hypernet.
\newcommand{\hyperVarGlobal}{1.5} %The variance of the hyperparameter distribution for sampling hyperparameters for training the hypernet.
\newcommand{\numHypernetParamsGlobal}{400,450} %This is \hyperDimSmall \cdot \numHiddenGlobal + \numHiddenGlobal + \numHiddenGlobal \cdot \numModelParams + \numModelParams.
\newcommand{\batchSizeLocal}{2} %The batch size for each update (in terms of hyperparameters).
\newcommand{\stepSizeHypernetLocal}{\stepSizeHypernetGlobal} %The step size to use when training the hypernet.
\newcommand{\hyperVarLocal}{0.00001} %The variance of the conditional hyperparameter distribution for sampling hyperparameters for training the hypernet.
\newcommand{\numHypernetParamsLocal}{15,700} %This is \hyperDimSmall \cdot \numModelParams + \numModelParams.
\newcommand{\numHypernetItersSmall}{10} %The number of iterations to use when training the hypernet.
\newcommand{\stepSizeHyperSmall}{0.1} %The step size to use when training the hyperparameter.
\newcommand{\numHyperItersSmall}{1} %The number of iterations to use when training the hypernet.
\newcommand{\hyperDimMedium}{10} %The dimensionality of the hyperparameter space for the medium dimensionality hyperparameter experiment.
\newcommand{\numTrainMedium}{50,000}
\newcommand{\minibatchMedium}{100}
\newcommand{\numValidMedium}{\numValidSmall}
\newcommand{\numTestMedium}{\numTestSmall}
\newcommand{\numHiddenMedium}{1} %The number of hidden units in the hypernet.
\newcommand{\batchSizeMedium}{\batchSizeGlobal} %The batch size for each update (in terms of hyperparameters).
\newcommand{\stepSizeHypernetMedium}{\stepSizeHypernetGlobal} %The step size to use when training the hypernet.
\newcommand{\hyperVarMedium}{\hyperVarLocal} %The variance of the conditional hyperparameter distribution for sampling hyperparameters for training the hypernet.
\newcommand{\numHypernetParamsMedium}{86,350} %This is \hyperDimSmall \cdot \numModelParams + \numModelParams.
\newcommand{\numHypernetItersMedium}{10} %The number of iterations to use when training the hypernet.
\newcommand{\numHyperItersMedium}{1} %The number of iterations to use when training the hyperparameter.
\newcommand{\stepSizeHyperMedium}{0.0001} %The step size to use when training the hyperparameter.
\newcommand{\hyperDimLarge}{\numModelParams} %The dimensionality of the hyperparameter space for the large dimensionality hyperparameter experiment.
\newcommand{\minibatchLarge}{\minibatchMedium}
\newcommand{\numTrainLarge}{\numValidMedium}
\newcommand{\numValidLarge}{\numValidMedium}
\newcommand{\numTestLarge}{\numTestMedium}
\newcommand{\numHiddenLarge}{10} %The number of hidden units in the hypernet.
\newcommand{\batchSizeLarge}{\batchSizeMedium} %The batch size for each update (in terms of hyperparameters).
\newcommand{\stepSizeHypernetLarge}{\stepSizeHypernetMedium} %The step size to use when training the hypernet.
\newcommand{\hyperVarLarge}{\hyperVarMedium} %The variance of the conditional hyperparameter distribution for sampling hyperparameters for training the hypernet.
\newcommand{\numHypernetParamsLarge}{164,860} %This is \hyperDimSmall \cdot \numModelParams + \numModelParams.
\newcommand{\numHypernetItersLarge}{\numHypernetItersMedium} %The number of iterations to use when training the hypernet.
\newcommand{\stepSizeHyperLarge}{\numHypernetItersMedium} %The step size to use when training the hyperparameter.
\newcommand{\numTrainVs}{25}
\newcommand{\numValidVs}{10,215}
\newcommand{\numHiddenLossVs}{\numModelParams}

\section{Experiments}
\label{sec.experiments}

\begin{figure}[ht!]
\centering
\includegraphics[width=0.5\textwidth]{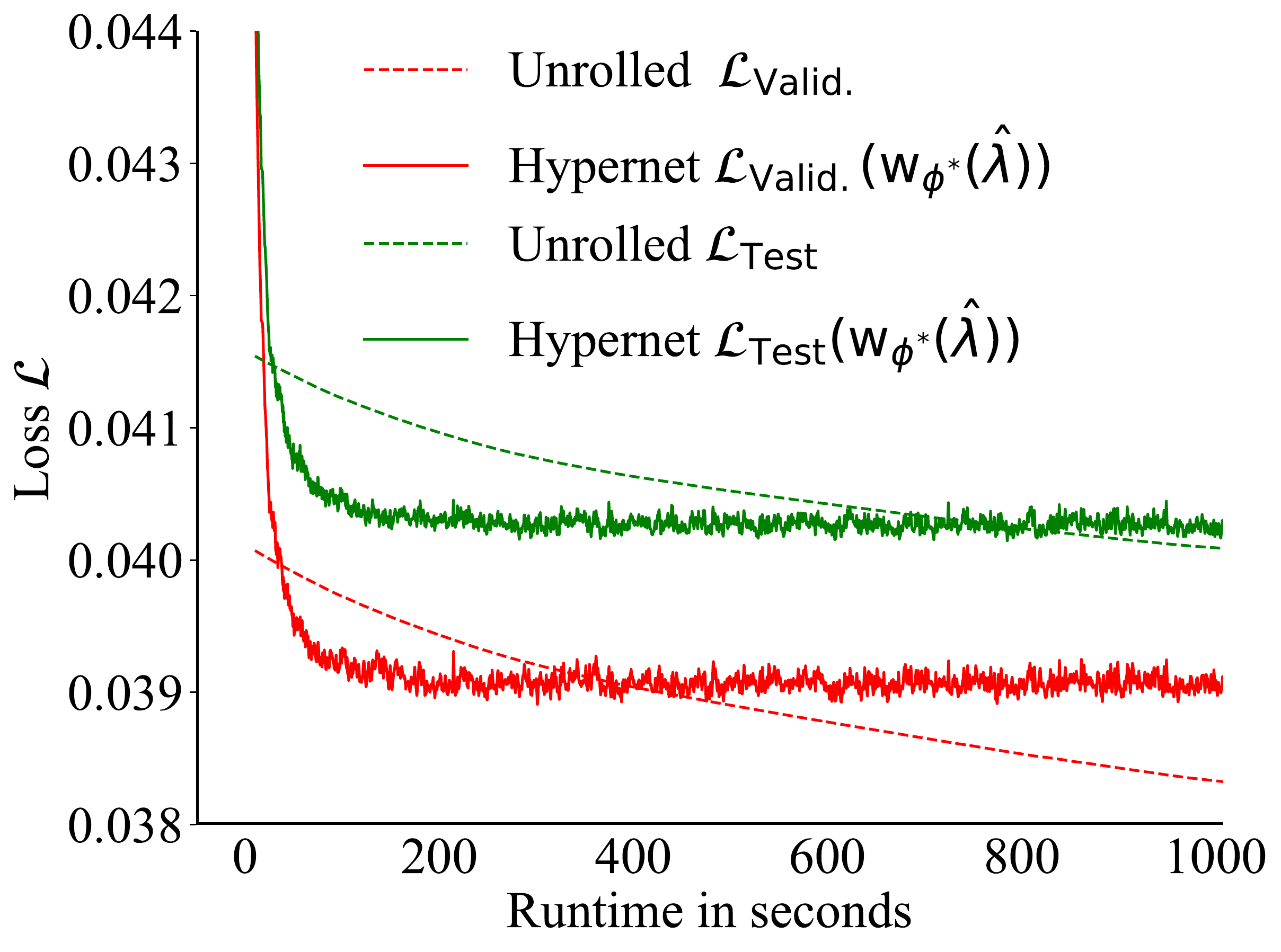}
\includegraphics[width=0.5\textwidth]{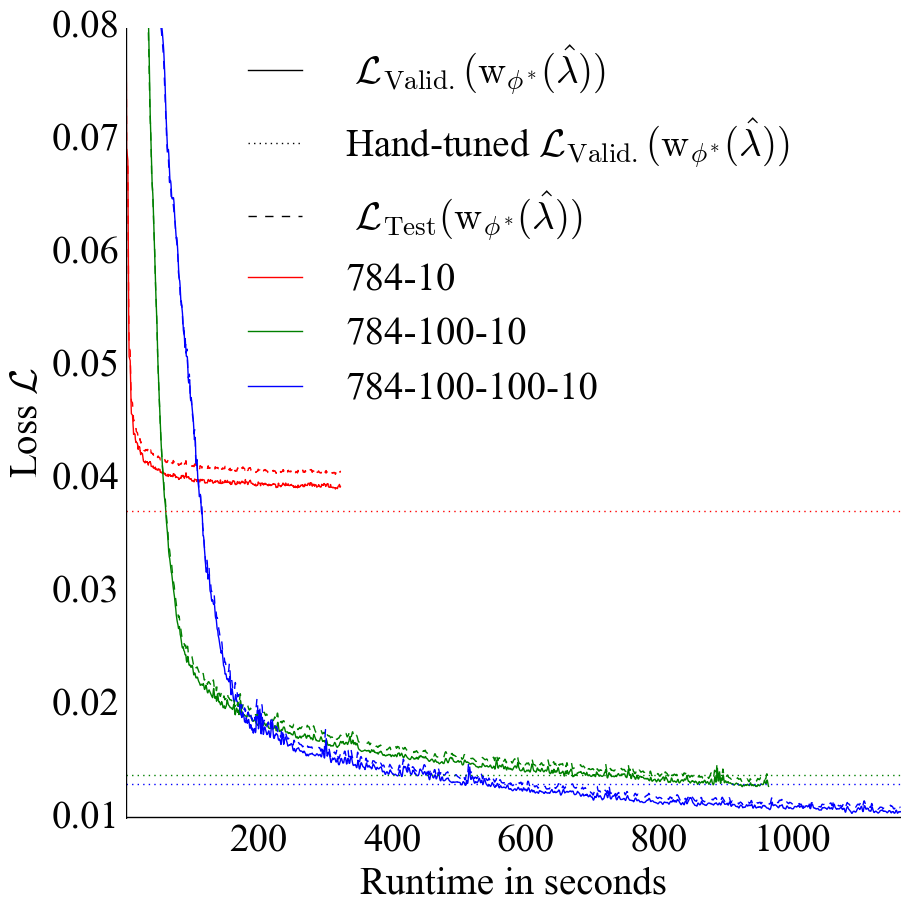}
\caption{
Validation and test losses during hyperparameter optimization with a separate $L_{2}$ weight decay applied to each weight in the model.
Thus, models with more parameters have more hyperparameters.
\emph{Top:} We solve the $7,850$-dimensional hyperparameter optimization problem with a linear model and multiple algorithms.
%The weights $\param_{\phi^{*}}$ are output by the hypernet for current hyperparameter $\curRename{\lambda}$, while random losses are for the best result of a random search.
Hypernetwork-based optimization converges to a sub-optimal solution faster than unrolled optimization from \citet{maclaurin2015gradient}.
%We also observe significant overfitting of the hyperparameters on the validation set, which may be reduced be introducing hyperhyperparameters (parameters of the hyperparameter prior).
\emph{Bottom:} Hyper-training is applied different layer configurations in the model.
%The hand-tuned regularization parameters on the 784-10, 784-100-10, and 784-100-100-10 models have a validation losses of $0.434, 0.157$ and $0.206$ respectively.
\label{fig:exp3}
}
\end{figure}
In our experiments, we examine the standard example of stochastic gradient-based optimization of neural networks, with a weight regularization penalty.
Some gradient-based methods explicitly use the gradient of a loss, while others use the gradient of a learned surrogate loss.
Hyper-training learns and substitutes a surrogate best-response function into a real loss.
We may contrast our algorithm with methods learning the loss like Bayesian optimization, gradient-based methods only handling hyperparameters that affect the training loss and gradient-based methods which can handle optimization parameters.
The best comparison for hyper-training is to gradient-based methods which only handle parameters affecting the training loss because other methods apply to a more general set of problems.
In this case, we write the training and validation losses as:

\begin{align*}
\innerLoss{\param}{\hyper} &= \innerLossEExpand{\param}{\hyper} \\
\outerLoss{\param} &= \outerLossEExpand{\param}
\end{align*}

In all experiments, Algorithms~\ref{algGlobal} or \ref{algLocal} are used to optimize weights with a mean squared error on MNIST~\citep{lecun1998gradient} with $\lossSymbol_{\mathrm{Reg}}$ as an $L_{2}$ weight decay penalty weighted by $\exp(\hyper)$.
The elementary model has $\numModelParams$ weights.
All hidden units in the hypernetwork have a ReLU activation~\citep{nair2010rectified} unless otherwise specified.
Autograd~\citep{maclaurin2015autograd} was used to compute all derivatives.
For each experiment, the minibatch samples $\batchSizeGlobal$ pairs of hyperparameters and up to $1,000$ training data points.
We used Adam for training the hypernetwork and hyperparameters, with a step size of $\stepSizeHypernetGlobal$.
We ran all experiments on a CPU.

\subsection{Learning a global best-response}
Our first experiment, shown in Figure~\ref{fig:exp1}, demonstrates learning a global approximation to a best-response function using Algorithm~\ref{algGlobal}.
To make visualization of the regularization loss easier, we use $\numTrainSmall$ training data points to exacerbate overfitting.
We compare the performance of weights output by the hypernetwork to those trained by standard cross-validation (Algorithm~\ref{alg1}).
Thus, elementary weights were randomly initialized for each hyperparameter choice and optimized using Adam~\citep{kingma2014adam} for $\realIter$ iterations with a step size of $\stepSizeReal$.

When training the hypernetwork, hyperparameters were sampled from a broad Gaussian distribution: $\hyperDist = \mathcal{N} ( \hyperMeanGlobal, \hyperVarGlobal )$.
The hypernetwork has $\numHiddenGlobal$ hidden units which results in $\numHypernetParamsGlobal$ parameters of the hypernetwork.

The minimum of the best-response in Figure~\ref{fig:exp1} is close to the real minimum of the validation loss, which shows a hypernetwork can satisfactorily approximate a global best-response function in small problems.

\subsection{Learning a local best-response}
Figure~\ref{fig:exp2} shows the same experiment, but using the Algorithm~\ref{algLocal}.
The fused updates result in finding a best-response approximation whose minimum is the actual minimum faster than the prior experiment.
The conditional hyperparameter distribution is given by $\hyperDistVar = \mathcal{N} ( \curRename{\hyperFixed}, 0.00001)$.
The hypernetwork is a linear model, with only $\numHypernetParamsLocal$ weights.
We use the same optimizer as the global best-response to update both the hypernetwork and the hyperparameters.

Again, the minimum of the best-response at the end of training minimizes the validation loss.
This experiment shows that using only a locally-trained linear best-response function can give sufficient gradient information to optimize hyperparameters on a small problem.
Algorithm~\ref{algLocal} is also less computationally expensive than Algorithms~\ref{alg1} or \ref{algGlobal}.

\begin{figure}[t!]
	%\centering
	\hspace{-0.02\textwidth}
	\begin{tabular}{@{\hskip3pt}c @{\hskip3pt}c @{\hskip3pt}c @{\hskip3pt}c} 
	&GP mean&Hyper-training fixed&Hyper-training\\
	\rotatebox{90}{\,\,\,\,\,\,\,\,\,\,\,\, Inferred loss}\hspace{-0.02\textwidth}&\includegraphics[width=0.14\textwidth]{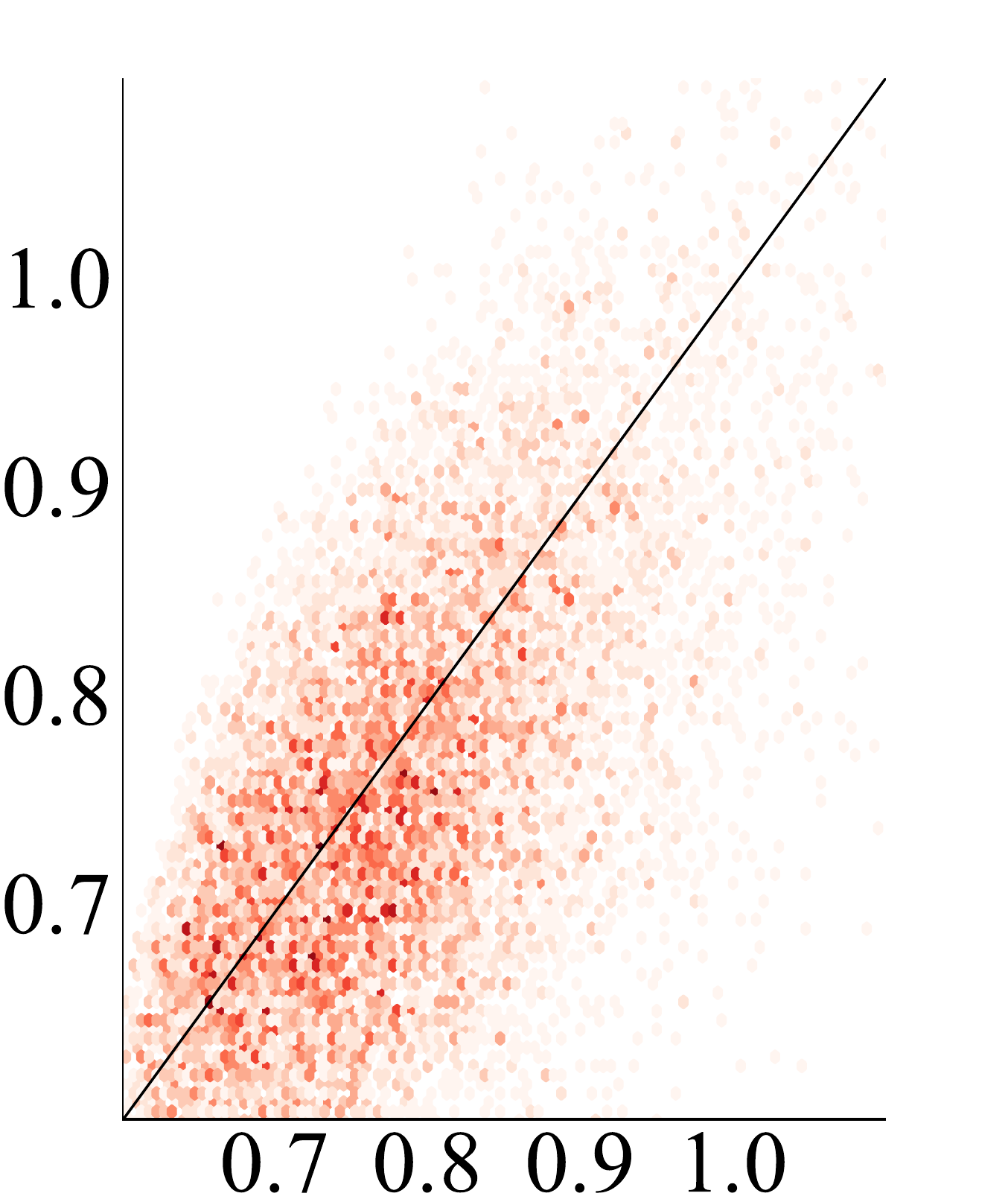}&\includegraphics[width=0.14\textwidth]{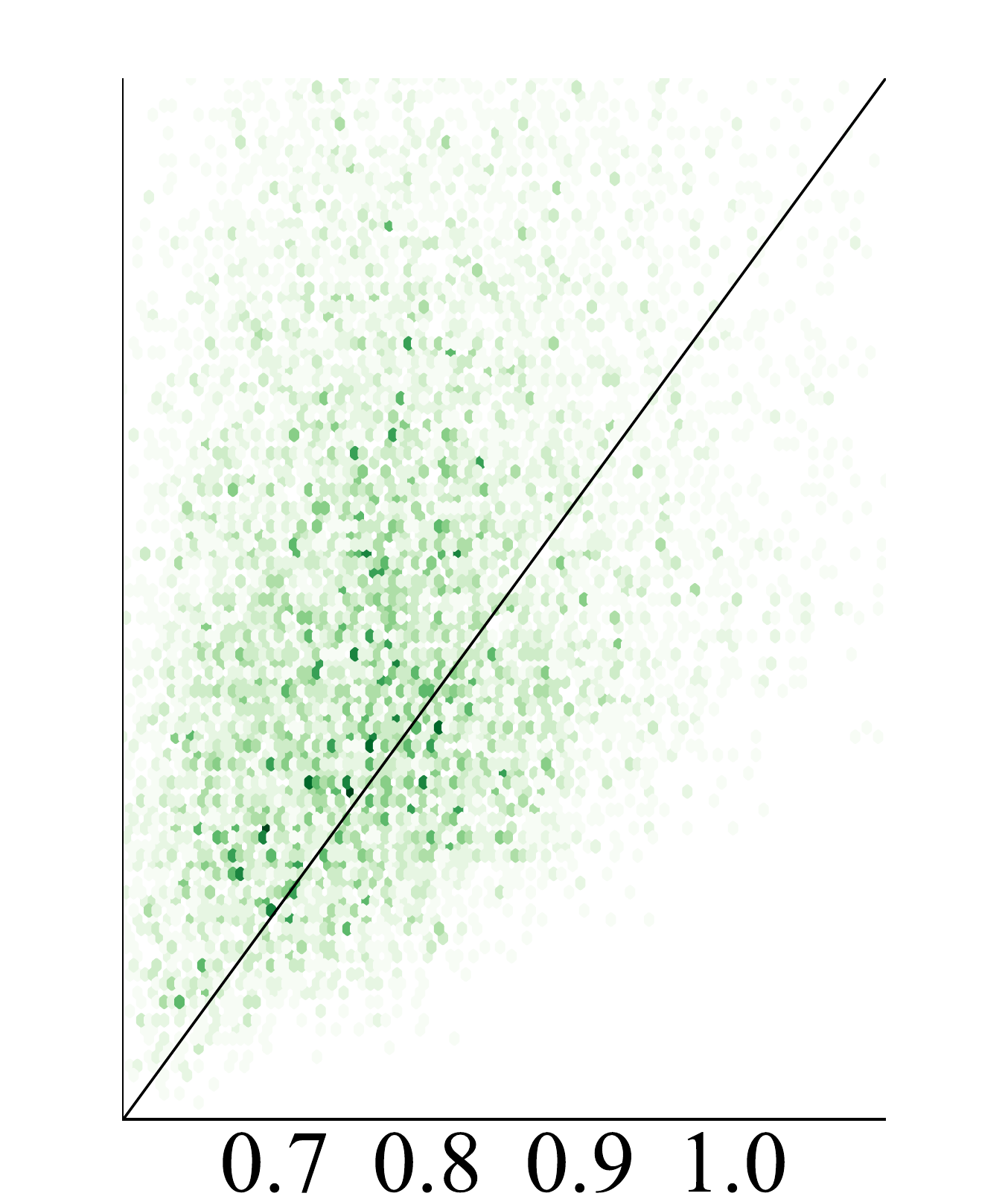}&\includegraphics[width=0.14\textwidth]{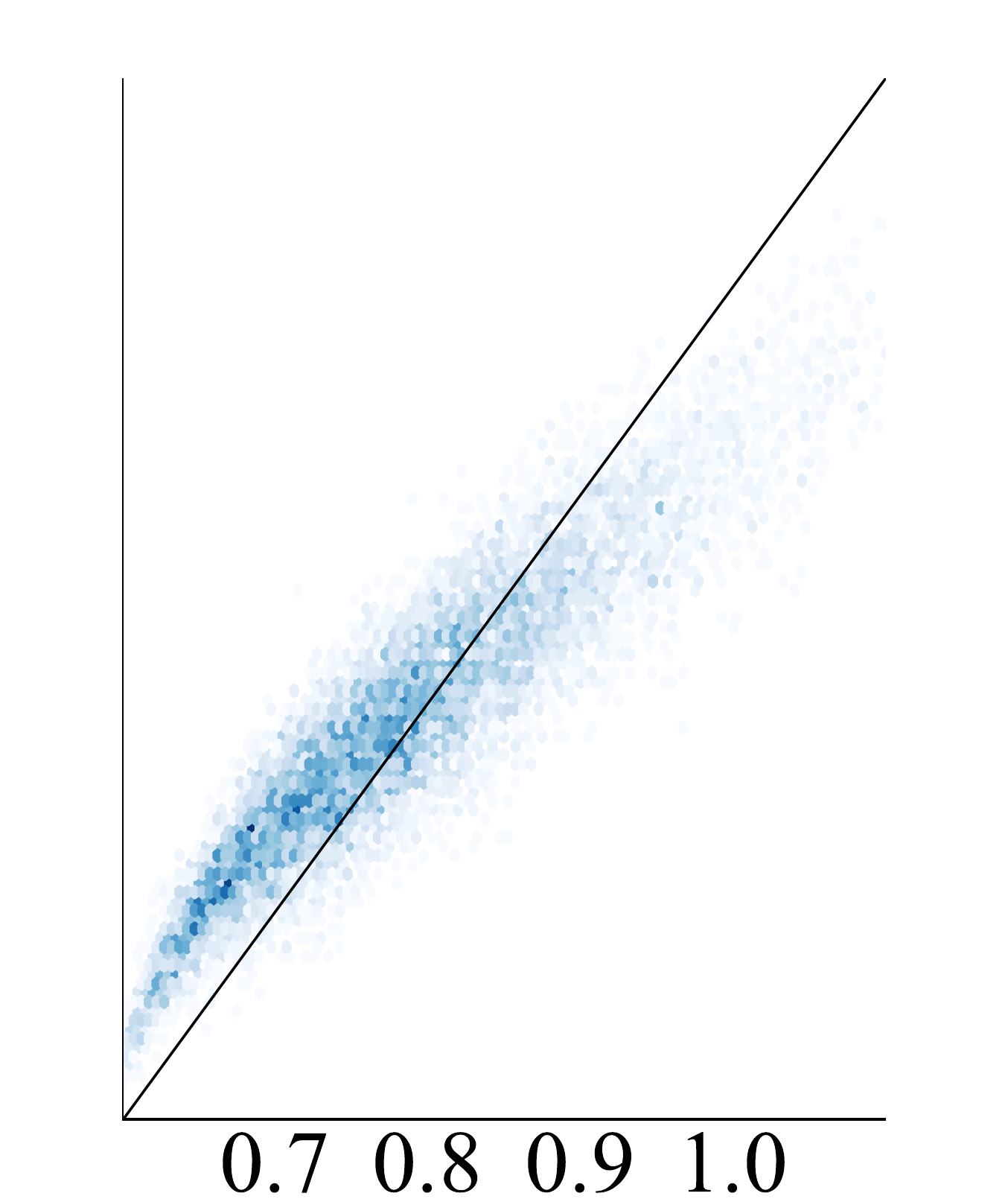}\\
	&&True loss&\\
	\rotatebox{90}{\,\,\,\,\,\,\,\,\,\,\,\,\,\,\, Frequency}&\includegraphics[width=0.14\textwidth]{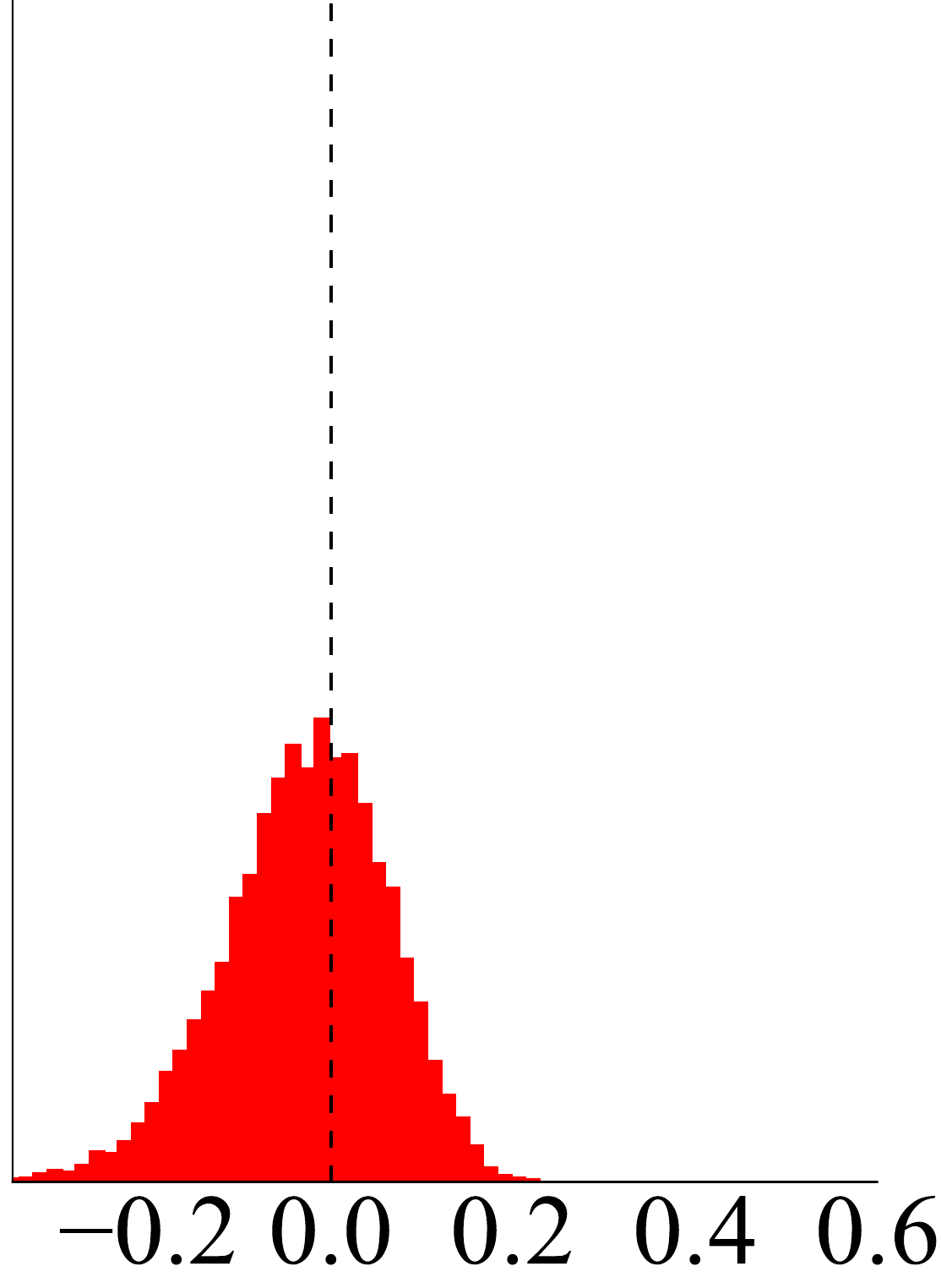}&\includegraphics[width=0.14\textwidth]{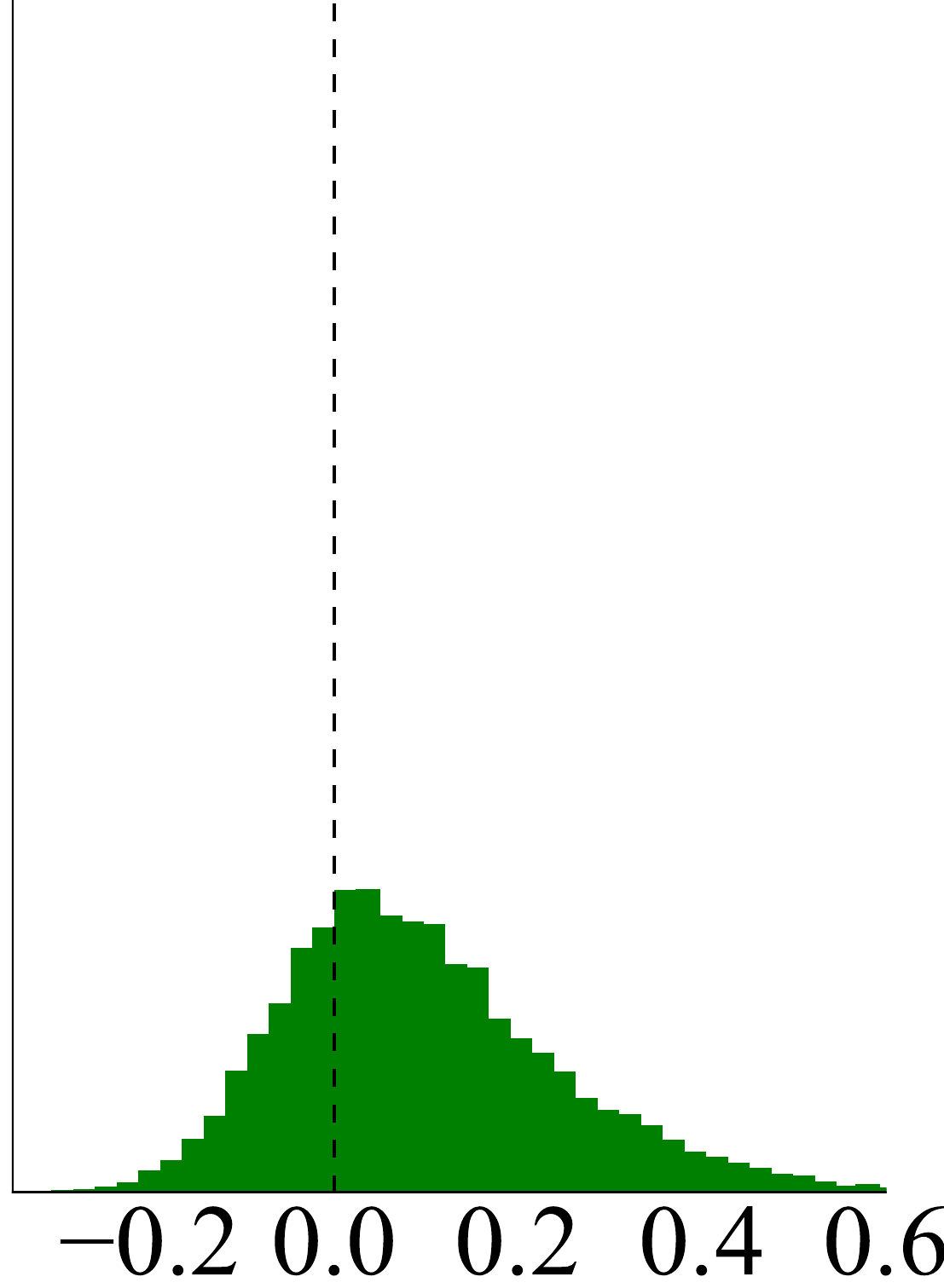}&\includegraphics[width=0.14\textwidth]{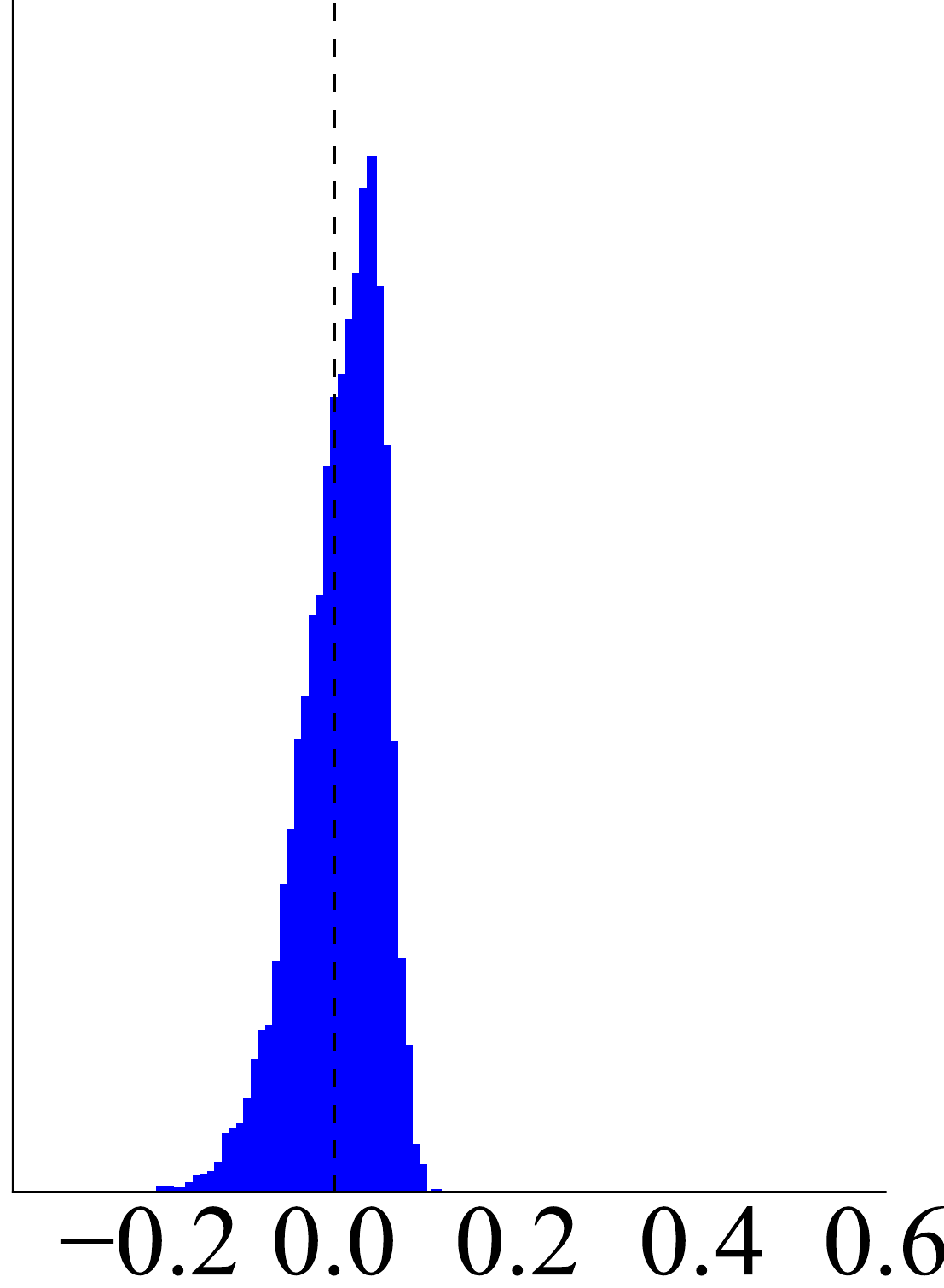}\\
	&&Inferred - true loss&\\
	\end{tabular}
	\caption{
	Comparing three approaches to inferring validation loss.
	\emph{First column:}
	{\color{red}A Gaussian process}, fit on $\numTrainVs$ hyperparameters and the corresponding validation losses.
	\emph{Second column:}
	{\color{green}A hypernetwork, fit on the same $\numTrainVs$ hyperparameters} and the corresponding optimized weights.
	\emph{Third column:}
	Our proposed method, {\color{blue}a hypernetwork trained with stochastically sampled hyperparameters}.
	\emph{Top row:}
	The distribution of inferred and true losses.
	The diagonal black line is where predicted loss equals true loss.
	\emph{Bottom row:}
	The distribution of differences between inferred and true losses.
	The Gaussian process often under-predicts the true loss, while the hypernetwork trained on the same data tends to over-predict the true loss.
	\label{fig:exp5}
	}
\end{figure}

\subsection{Hyper-training and unrolled optimization}
To compare hyper-training with other gradient-based hyperparameter optimization methods, we train models with $\hyperDimLarge$ hyperparameters and a separate $L_{2}$ weight decay applied to each weight in a 1 layer (linear) model.
The conditional hyperparameter distribution and optimizer for the hypernetwork and hyperparameters is the same the prior experiment.
We factorize the weights for the model by selecting a hypernetwork with $\numHiddenLarge$ hidden units.
The factorized linear hypernetwork has $\numHiddenLarge$ hidden units giving $\numHypernetParamsLarge$ weights.
Each hypernetwork iteration is $2 \cdot \numHiddenLarge$ times as expensive as an iteration on just the model because there is the same number of hyperparameters as model parameters.

Figure~\ref{fig:exp3}, top, shows that Algorithm~\ref{algLocal} converges more quickly than the unrolled reverse-mode optimization introduced in \citet{maclaurin2015gradient} and implemented by \citet{franceschi2017forward}.
Hyper-training reaches sub-optimal solutions because of limitations on how many hyperparameters can be sampled for each update but overfits validation data less than unrolling.
Standard Bayesian optimization cannot be scaled to this many hyperparameters.
Thus, this experiment shows Algorithm~\ref{algLocal} can efficiently partially optimize thousands of hyperparameters.
It may be useful to combine these methods by using a hypernetwork to output initial parameters and then unrolling several steps of optimization to differentiate through.

\subsection{Optimizing with deeper networks}
To see if we can optimize deeper networks with hyper-training we optimize models with 1, 2, and 3 layers and a separate $L_{2}$ weight decay applied to each weight.
The conditional hyperparameter distribution and optimizer for the hypernetwork and hyperparameters is the same the prior experiment.
We factorize the weights for each model by selecting a hypernetwork with $\numHiddenLarge$ hidden units.
%We compare with a fixed regularization of hand-tuned hyperparameters.

Figure~\ref{fig:exp3}, bottom, shows that Algorithm~\ref{algLocal} can scale to networks with multiple hidden layers and outperform hand-tuned settings.
As we add more layers the difference between validation loss and testing loss decreases, and the model performs better on the validation set.
Future work should compare other architectures like recurrent or convolutional networks.
Additionally, note that more layers perform with lesser training (not shown), validation, and test losses, instead of lower training loss and higher validation or test loss.
This performance indicates that using weight decay on each weight could be a prior for generalization, or that hyper-training enforces another useful prior like the continuity of a best-response.
%
%\begin{wrapfigure}[24]{r}{0.45\textwidth}
%	\vspace{-0.35cm}
%	\includegraphics[width=0.45\textwidth]{hypernets_local_large.pdf}
%	\label{fig:exp4}
%	\vspace{-0.73cm}
%	\caption{
%	Validation and testing results from doing the $\hyperDimLarge$-dimensional hyperparameter optimization.
%	A separate $L_{2}$ weight regularization is applied to the weights the model uses for each class.
%	The weights $\param_{\phi^{*}}$ are output by the hypernet for the current hyperparameter $\hat{\lambda}$, while the random losses are for the best result from a random search.
%	}
%\end{wrapfigure}
%
%
%\begin{figure}
%\hspace{-0.05\textwidth}

\subsection{Estimating weights versus estimating loss}
Our approach differs from Bayesian optimization which attempts to directly model the validation loss of optimized weights, where we try to learn to predict optimal weights.
In this experiment, we untangle the reason for the better performance of our method:
Is it because of a better inductive bias, or because our way can see more hyperparameter settings during optimization?
%learning optimized weights is compared to learning optimized loss for the predicted loss on held-out hyperparameters.
%Also, sampling hyperparameters is compared with using a fixed set of hyperparameters.
%This allows a decomposition of the benefits of hyper-training.

First, we constructed a hyper-training set: We optimized $\numTrainVs$ sets of weights to completion, given randomly-sampled hyperparameters.
We chose $\numTrainVs$ samples since that is the regime in which we expect Gaussian process-based approaches to have the most significant advantage.
We also constructed a validation set of $\numValidVs$ (optimized weight, hyperparameter) tuples generated in the same manner.
%tuples for validation are generated with the same optimizer parameters as the global experiment.
{\color{red}We then fit a Gaussian process (GP)} regression model with an RBF kernel from sklearn on the validation loss data.
{\color{green}A hypernetwork is fit to the same set of hyperparameters and data}.
%However, this hypernetwork was trained to fit optimized training weights, not optimized validation loss.
Finally, {\color{blue}we optimize another hypernetwork using Algorithm~\ref{algGlobal}}, for the same amount of time as building the GP training set.
The two hypernetworks were linear models and trained with the same optimizer parameters as the $\hyperDimLarge$-dimensional hyperparameter optimization.

Figure~\ref{fig:exp5} shows the distribution of prediction errors of these three models.
We can see that the Gaussian process tends to underestimate loss.
The hypernetwork trained with the same small fixed set of examples tends to overestimate loss.
We conjecture that this is due to the hypernetwork producing bad weights in regions where it doesn't have enough training data.
Because the hypernetwork must provide actual weights to predict the validation loss, poorly-fit regions will overestimate the validation loss.
Finally, the hypernetwork trained with Algorithm~\ref{algGlobal} produces errors tightly centered around 0.
The main takeaway from this experiment is a hypernetwork can learn more accurate surrogate functions than a GP for equal compute budgets because it views (noisy) evaluations of more points.
%Also, if a point is explored which minimizes the predicted validation loss, then hyper-training explores points which perform better than Bayesian optimization.
%
%Table~\ref{gp_table} shows how varying the number of training tuples affects the hyperparameter which minimizes the predicted loss, where fixed input hyper-training uses the same fixed inputs as the Gaussian process.
%Algorithm~\ref{algGlobal} consistently identifies hyperparameters with a better true performance than the other two approaches.

%Code for all experiments will be made available upon publication.

%\begin{table}
%\begin{tabular}{r | c c c c c}
% & \multicolumn{5}{c}{Evaluations of Validation Loss}\\
% Method & 10 & 25 & 100 & 250 &1000\\
 %\midrule
 %{\color{red}Gaussian process} & 0.90 & 0.67 & 0.60 & 0.60 & 0.62\\
 %{\color{green}hypernetwork trained on same evaluations} & 0.65 & \textbf{0.60} & \textbf{0.59} & \textbf{0.59} & \textbf{0.59}\\
 %{\color{blue}hypernetwork trained stochastically for equivalent time} & \textbf{0.60} & 0.61 & \textbf{0.59} & \textbf{0.59} & \textbf{0.59}
 %Actual best hyperparameters & 0.593 & 0.593& 0.593 & 0.593 & 0.593
%\end{tabular}
%\caption{Actual validation loss at the best-predicted hyperparameter setting, according to each model.
%To estimate the true best validation loss, a random search was run for 10,000 hyperparameters, which did not find a better hyperparameter than a sampled hypernetwork.
%\label{gp_table}}
%\vspace{-0.5cm}
%\end{table}

\section{Conclusions and Future Work}
In this paper, we addressed the question of tuning hyperparameters using gradient-based optimization, by replacing the training optimization loop with a differentiable hypernetwork.
%Furthermore, we 
%\paragraph{Contributions}
%The following contributions are made for continuous hyperparameters:
%\begin{itemize}
%\item Presented algorithms that efficiently learn a differentiable approximation to a best-response without nested optimization.
We gave a theoretical justification that sufficiently large networks will learn the best-response for all hyperparameters viewed in training.
We also presented a simpler and more scalable method that jointly optimizes both hyperparameters and hypernetwork weights, allowing our method to work with manageably-sized hypernetworks.

Experimentally, we showed that hypernetworks could provide a better inductive bias for hyperparameter optimization than Gaussian processes fitting the validation loss empirically.
%\end{itemize}

There are many directions to extend the proposed methods.
%Non-spherical conditional hyperparameter distributions can be explored, because it may drastically affect the speed a best-response is learned at.
%Maximum a priori (MAP) estimates for weights and maximum likelihood estimates (MLE) for hyperparameters are learned, but it may be possible to learn a distribution of weights for each hyperparameter, or a distribution of hyperparameters for Bayesian learning~\citep{neal2012bayesian}.
For instance, the hypernetwork could be composed with several iterations of optimization, as an easily-differentiable fine-tuning step.
Or, hypernetworks could be incorporated into meta-learning schemes, such as MAML~\citep{finn2017model}, which finds weights that perform a variety of tasks after unrolling gradient descent.

We also note that the prospect of optimizing thousands of hyperparameters raises the question of \emph{hyper-regularization}, or regularization of hyperparameters.
%Additionally, it may be useful to introduce hyperhyperparameters (parameters of the hyperparameter prior) for hyperparameters.
%

%We hope this initial exploration of stochastic hyperparameter optimization will inspire further refinements, such as hyper-regularization methods, or uncertainty-aware exploration using Bayesian hypernetworks.

\subsubsection*{Acknowledgments}
We thank Matthew MacKay, Dougal Maclaurin, Daniel Flam-Shepard, Daniel Roy, and Jack Klys for helpful discussions.

\bibliography{main_bib}
\bibliographystyle{icml2018}
\end{document}